\documentclass[letterpaper]{article}
\usepackage{aaai17}
\usepackage{times}
\usepackage{helvet}
\usepackage{courier}
\usepackage{amsmath}
\usepackage{amssymb}
\usepackage{url}

\newcommand\tagthis{\addtocounter{equation}{1}\tag{\theequation}}
\def\taglabel#1{\tagthis\label{#1}}

\newtheorem{prop}{Proposition}
\newtheorem{lemma}{Lemma}
\newtheorem{thm}{Theorem}
\newtheorem{example}{Example}

\newtheorem{cor}{Corollary}

\long\def\BOC#1\EOC{\message{(Commented text )}}
\long\def\BOCC#1\EOCC{\message{(Commented text )}}
\long\def\BOCCC#1\EOCCC{\message{(Commented text )}}
\long\def\optional#1{\empty}
\long\def\NB#1{}

\def\mvis{\!=\!}
\def\o{\overline}
\def\bi{\begin{itemize}}
\def\ii{\item}
\def\ei{\end{itemize}}
\def\beq{\begin{equation}}
\def\eeq#1{\label{#1}\end{equation}}
\def\ba{\begin{array}}
\def\ea{\end{array}}
\def\i#1{\hbox{\it #1\/}}

\def\sm{\hbox{\rm SM}}
\def\lpmln{\hbox{\rm LP}^{\rm{MLN}}}
\def\lpmln{{\rm LP}^{\rm{MLN}}}
\def\no{\i{not}}
\def\sneg{\sim\!\!}
\def\ar{\leftarrow}
\def\rar{\rightarrow}

\def\no{\i{not}}
\def\false{\hbox{\bf f}}
\def\true{\hbox{\bf t}}

\def\i#1{\hbox{\itshape #1\/}}

\def\qed{\quad \vrule height7.5pt width4.17pt depth0pt \medskip}

\begin{document}

\title{$\lpmln$, Weak Constraints, and P-log}

\author{Joohyung Lee and Zhun Yang \\
School of Computing, Informatics and Decision Systems Engineering \\
Arizona State University, Tempe, USA \\
{\tt \{joolee, zyang90\}@asu.edu}}

\maketitle

\begin{abstract}
$\lpmln$ is a recently introduced formalism that extends answer set programs by adopting the log-linear weight scheme of Markov Logic. This paper investigates the relationships between $\lpmln$ and two other extensions of answer set programs: weak constraints to express a quantitative preference among answer sets, and P-log to incorporate probabilistic uncertainty. We present a translation of $\lpmln$ into programs with weak constraints and a translation of P-log into $\lpmln$, which complement the existing translations in the opposite directions. The first translation allows us to compute the most probable stable models (i.e., MAP estimates) of $\lpmln$ programs using standard ASP solvers. This result can be extended to other formalisms, such as Markov Logic, ProbLog, and Pearl's Causal Models, that are shown to be translatable into $\lpmln$. The second translation tells us how probabilistic nonmonotonicity (the ability of the reasoner to change his probabilistic model as a result of new information) of P-log can be represented in $\lpmln$, which yields a way to compute P-log using standard ASP solvers and MLN solvers.
\end{abstract}

\section{Introduction} \label{sec:intro}

$\lpmln$ \cite{lee16weighted} is a recently introduced probabilistic logic programming language that extends answer set programs \cite{gel88} with the concept of weighted rules, whose weight scheme is adopted from that of Markov Logic \cite{richardson06markov}. It is shown in~\cite{lee16weighted,lee15markov} that $\lpmln$ is expressive enough to embed Markov Logic and several other probabilistic logic languages, such as ProbLog \cite{deraedt07problog}, Pearls' Causal Models  \cite{pearl00causality}, and a fragment of P-log \cite{baral09probabilistic}. 

Among several extensions of answer set programs to overcome the deterministic nature of the stable model semantics, $\lpmln$ is one of the few languages that incorporate the concept of weights into the semantics.  Another one is weak constraints \cite{buccafurri00enhancing}, which are to assign a quantitative preference over the stable models of non-weak constraint rules: weak constraints cannot be used for deriving stable models. It is relatively a simple extension of the stable model semantics but has turned out to be useful in many practical applications. Weak constraints are part of the ASP Core~2 language \cite{calimeri13aspcore2}, and are implemented in standard ASP solvers, such as {\sc clingo} and {\sc dlv}. 

P-log is a probabilistic extension of answer set programs. In contrast to weak constraints, it is highly structured and has quite a sophisticated semantics. One of its distinct features is {\em probabilistic nonmonotonicity} (the ability of the reasoner to change his probabilistic model as a result of new information) whereas, in most other probabilistic logic languages, new information can only cause the reasoner to abandon some of his possible worlds, making the effect of an update {\em monotonic}.

This paper reveals interesting relationships between $\lpmln$ and these two extensions of answer set programs. It shows how different weight schemes of $\lpmln$ and weak constraints are related, and how the probabilistic reasoning in P-log can be expressed in $\lpmln$. The result helps us understand these languages better as well as other related languages, and also provides new, effective computational methods based on the translations.

It is shown in~\cite{lee16weighted} that programs with weak constraints can be easily viewed as a special case of $\lpmln$ programs. In the first part of this paper, we show that an inverse translation is also possible under certain conditions, i.e., an $\lpmln$ program can be turned into a usual ASP program with weak constraints so that the most probable stable models of the $\lpmln$ program are exactly the optimal stable models of the program with weak constraints. The result allows for using ASP solvers for computing Maximum A Posteriori probability (MAP) estimates of $\lpmln$ programs. Interestingly, the translation is quite simple so it can be easily applied in practice.
Further, the result implies that MAP inference in other probabilistic logic languages, such as Markov Logic, ProbLog, and Pearl's Causal Models, can be computed by standard ASP solvers because they are known to be embeddable in $\lpmln$,  thereby allowing us to apply combinatorial optimization in standard ASP solvers to MAP inference in these languages.

In the second part of the paper, we show how P-log can be completely characterized in $\lpmln$. 
Unlike the translation in \cite{lee16weighted}, which was limited to a subset of P-log  that does not allow dynamic default probability, our translation applies to full P-log and complements the recent translation from $\lpmln$ into P-log in~\cite{balai16ontherelationship}. In conjunction with the embedding of $\lpmln$ in answer set programs with weak constraints, our work shows how MAP estimates of P-log can be computed by standard ASP solvers, which provides a highly efficient way to compute P-log.

\section{Preliminaries} \label{sec:prelim}

\subsection{Review: $\lpmln$} \label{sec:lpmln}

We review the definition of $\lpmln$ from~\cite{lee16weighted}.
In fact, we consider a more general syntax of programs than the one from~\cite{lee16weighted}, but this is not an essential extension. We follow the view of~\cite{ferraris11stable} by identifying logic program rules as a special case of first-order formulas under the stable model semantics. For example, rule $r(x)\ar p(x), \no\ q(x)$ is identified with formula $\forall x (p(x)\land\neg q(x)\rar r(x))$.
An $\lpmln$ program is a finite set of weighted first-order formulas $w: F$ where $w$ is a real number (in which case the weighted formula is called {\em soft}) or $\alpha$ for denoting the infinite weight (in which case it is called {\em hard}). An $\lpmln$ program is called {\em ground} if its formulas contain no variables. We assume a finite Herbrand Universe. Any $\lpmln$ program can be turned into a ground program by replacing the quantifiers with multiple conjunctions and disjunctions over the Herbrand Universe. Each of the ground instances of a formula with free variables receives the same weight as the original formula.


For any ground $\lpmln$ program $\Pi$ and any interpretation~$I$, 
$\o{\Pi}$ denotes the unweighted formula obtained from $\Pi$, and
${\Pi}_I$ denotes the set of $w: F$ in $\Pi$ such that $I\models F$, 
and $\sm[\Pi]$ denotes the set $\{I \mid \text{$I$ is a stable model of $\o{\Pi_I}$}\}$ (We refer the reader to the stable model semantics of first-order formulas in~\cite{ferraris11stable}). 
The {\em unnormalized weight} of an interpretation $I$ under $\Pi$ is defined as 
\[
\small
 W_\Pi(I) =
\begin{cases}
  exp\Bigg(\sum\limits_{w:F\;\in\; {\Pi}_I} w\Bigg) & 
      \text{if $I\in\sm[\Pi]$}; \\
  0 & \text{otherwise}. 
\end{cases}
\]
The {\em normalized weight} (a.k.a. {\em probability}) of an interpretation $I$ under~$\Pi$ is defined as  \[ 
\small 
  P_\Pi(I)  = 
  \lim\limits_{\alpha\to\infty} \frac{W_\Pi(I)}{\sum\limits_{J\in {\rm SM}[\Pi]}{W_\Pi(J)}}. 
\] 
$I$ is called a {\sl (probabilistic) stable model} of $\Pi$ if $P_\Pi(I)\ne 0$.

\subsection{Review: Weak Constraints} \label{ssec:weak} 

A {\em weak constraint} has the form 
\[
  {\tt :\sim}\ F \ \ \  [\i{Weight}\ @\ \i{Level}]. 
\] 
where $F$ is a ground formula, $\i{Weight}$ is a real number and $\i{Level}$ is a nonnegative integer. Note that the syntax is more general than the one from the literature \cite{buccafurri00enhancing,calimeri13aspcore2}, where $F$ was restricted to conjunctions of literals.\footnote{
A literal is either an atom $p$ or its negation $\no\ p$.} 
We will see the generalization is more convenient for stating our result, but will also present translations that conform to the restrictions imposed on the input language of ASP solvers.

Let $\Pi$ be a program $\Pi_1\cup\Pi_2$, where $\Pi_1$ is a set of ground formulas and $\Pi_2$ is a set of weak constraints. We call $I$ a stable model of $\Pi$ if it is a stable model of $\Pi_1$ (in the sense of \cite{ferraris11stable}).
For every stable model $I$ of $\Pi$ and any nonnegative integer $l$, the {\em penalty} of $I$ at level~$l$, denoted by $\i{Penalty}_\Pi(I,l)$,  is defined as 
\[
\sum\limits_{:\sim\ F [w@l]\in\Pi_2,\atop I\models F}  w .
\]
For any two stable models $I$ and $I'$ of~$\Pi$, we say $I$ is {\em dominated} by $I'$ if 
\bi
\ii there is some nonnegative integer $l$ such that $\i{Penalty}_\Pi(I',l) < \i{Penalty}_\Pi(I,l)$ and 
\ii for all integers $k>l$, $\i{Penalty}_\Pi(I',k) = \i{Penalty}_\Pi(I,k)$.
\ei
A stable model of $\Pi$ is called {\em optimal} if it is not dominated by another stable model of $\Pi$.

\section{Turning $\lpmln$ into Programs with Weak Constraints} \label{sec:lpmln2wc}

\subsection{General Translation} 

We define a translation that turns an $\lpmln$ program into a program with weak constraints.
For any ground $\lpmln$ program $\Pi$, the translation ${\rm lpmln2wc}(\Pi)$ is simply defined as follows. 
We turn each weighted formula $w: F$ in $\Pi$ into $\{F\}^{\rm ch}$, where $\{F\}^{\rm ch}$ is a choice formula, standing for $F\lor \neg F$ \cite{ferraris11stable}. Further, we add
\beq
   :\sim F\ \ \ [-1@1]
\eeq{hard-wc}
if $w$ is $\alpha$, and 
\beq
   :\sim F\ \ \ [-w@0]
\eeq{soft-wc}
otherwise. 

Intuitively, choice formula $\{F\}^{\rm ch}$ allows $F$ to be either included or not in deriving a stable model.\footnote{This view of choice formulas was used in PrASP \cite{nickles14probabilistic} in defining {\sl spanning} programs.}  
When $F$ is included, the stable model gets the (negative) penalty $-1$ at level $1$ or $- w$ at level $0$ depending on the weight of the formula, which corresponds to the (positive) ``reward'' $e^\alpha$ or $e^{w}$ that it receives under the $\lpmln$ semantics.  

The following proposition tells us that choice formulas can be used for generating the members of $\sm[\Pi]$. 

\begin{prop}\label{prop:lpmln2wc}
For any $\lpmln$ program $\Pi$, the set $\sm[\Pi]$ is exactly the set of the stable models of ${\rm lpmln2wc(\Pi)}$.
\end{prop}

The following theorem follows from~Proposition~\ref{prop:lpmln2wc}. As the probability of a stable model of an $\lpmln$ program $\Pi$ increases, the penalty of the corresponding stable model of ${\rm lpmln2wc}(\Pi)$ decreases, and the distinction between hard rules and soft rules can be simulated by the different levels of weak constraints, and vice versa.

\begin{thm}\label{thm:lpmln2wc}
For any $\lpmln$ program~$\Pi$, the most probable stable models (i.e., the stable models with the highest probability) of $\Pi$ are precisely the optimal stable models of the program with weak constraints ${\rm lpmln2wc}(\Pi)$.
\end{thm}

\begin{example}\label{ex:1}
For program $\Pi$:
\beq
\ba {rl|crl}
  10:& \ \  p\rar q  \ \  \ &  &   5:& \ \ p \\
   1: & \ \  p\rar r  &  &  -20:&   \neg r\rar\bot 
\ea
\eeq{ex1}
$\sm[\Pi]$ has $5$ elements: $\emptyset$, $\{p\}$, $\{p,q\}$, $\{p,r\}$, $\{p,q,r\}$. 
Among them, $\{p,q\}$ is the most probable stable model, whose weight is $e^{15}$, while $\{p,q,r\}$ is a probabilistic stable model whose weight is $e^{-4}$.
The translation yields 
\[
\ba {c|cccrcc}
   \{p\rar q\}^{\rm ch}            &   &    :\sim &     p\rar q\ \              &[-10\ @\ 0]   \\
   \{p\rar r\}^{\rm ch}             &   &    :\sim &    p\rar r\ \               & [-1\ @\ 0] \\
   \{p\}^{\rm ch}                     &   &    :\sim &    p   \ \                     & [-5\ @\ 0] \\
   \{\neg r\rar\bot\}^{\rm ch} &   &    :\sim &    \neg r\rar\bot   \ \ &      [20\ @\ 0] 
\ea
\]
whose optimal stable model is $\{p,q\}$ with the penalty at level $0$ being $-15$, while $\{p,q,r\}$ is a stable model whose penalty at level $0$ is $4$.
\end{example}

The following example illustrates how the translation accounts for the difference between hard rules and soft rules by assigning different levels. 

\begin{example}\label{ex:bird}
Consider the $\lpmln$ program $\Pi$ in Example~1 from \cite{lee16weighted}. 
{\small
\[
\ba {llr}
\alpha: & \i{Bird}(\i{Jo}) \ar \i{ResidentBird}(\i{Jo}) & (r1) \\
\alpha: & \i{Bird}(\i{Jo}) \ar \i{MigratoryBird}(\i{Jo})  & (r2)\\
\alpha: &  \bot \ar \i{ResidentBird}(\i{Jo}), \i{MigratoryBird}(\i{Jo}) & (r3)  \\
2: & \i{ResidentBird}(\i{Jo}) & (r4) \\ 
1: & \i{MigratoryBird}(\i{Jo}) & (r5)
\ea
\]
}
The translation ${\rm lpmln2wc}(\Pi)$ is \footnote{Recall that we identify the rules with the corresponding first-order formulas.}
{\small
\[
\ba {lllr}
\{ \i{Bird}(\i{Jo}) \ar \i{ResidentBird}(\i{Jo}) \}^{\rm ch}  \\
\{ \i{Bird}(\i{Jo}) \ar \i{MigratoryBird}(\i{Jo}) \}^{\rm ch}  \\
 \{ \bot \ar \i{ResidentBird}(\i{Jo}), \i{MigratoryBird}(\i{Jo})\}^{\rm ch}  \\ 
\{ \i{ResidentBird}(\i{Jo}) \}^{\rm ch} \\ 
\{ \i{MigratoryBird}(\i{Jo}) \}^{\rm ch} \\[0.7em]
\hspace{0em} :\sim \i{Bird}(\i{Jo}) \ar \i{ResidentBird}(\i{Jo}) & [-1@1] \\
:\sim \i{Bird}(\i{Jo}) \ar \i{MigratoryBird}(\i{Jo})  & [-1@1] \\
:\sim \bot \ar \i{ResidentBird}(\i{Jo}), \i{MigratoryBird}(\i{Jo})\} & [-1@1]   \\ 
:\sim \i{ResidentBird}(\i{Jo}) & [-2@0]  \\ 
:\sim \i{MigratoryBird}(\i{Jo}) & [-1@0] \\ 
\ea
\]
}

The three probabilistic stable models of $\Pi$,  
$\emptyset$, 
$\{\i{Bird}(Jo), \i{ResidentBird}(Jo)\}$,  and
$\{\i{Bird}(Jo), \i{MigratoryBird}(Jo)\}$,  
have the same penalty $-3$ at level $1$.  Among them, $\{\i{Bird}(Jo), \i{ResidentBird}(Jo)\}$  has the least penalty at level $0$, and thus is an optimal stable model of ${\rm lpmln2wc}(\Pi)$.
\end{example}

In some applications, one may not want any hard rules to be violated assuming that hard rules encode definite knowledge. For that,  ${\rm lpmln2wc}(\Pi)$ can be modified by simply turning hard rules into the usual ASP rules. Then the stable models of  ${\rm lpmln2wc}(\Pi)$ satisfy all hard rules. 
For example,  the program in Example~\ref{ex:bird} can be translated into programs with weak constraints as follows. 
{\small
\[
\ba {lllr}
\i{Bird}(\i{Jo}) \ar \i{ResidentBird}(\i{Jo})   \\
\i{Bird}(\i{Jo}) \ar \i{MigratoryBird}(\i{Jo})   \\
\bot \ar \i{ResidentBird}(\i{Jo}), \i{MigratoryBird}(\i{Jo})  \\ 
\{ \i{ResidentBird}(\i{Jo}) \}^{\rm ch} \\ 
\{ \i{MigratoryBird}(\i{Jo}) \}^{\rm ch} \\[0.7em]
:\sim \i{ResidentBird}(\i{Jo}) & [-2@0]  \\ 
:\sim \i{MigratoryBird}(\i{Jo}) & [-1@0] 
\ea
\]
}

Also note that while the most probable stable models of $\Pi$ and the optimal stable models of ${\rm lpmln2wc}(\Pi)$ coincide, their weights and penalties are not proportional to each other. The former is defined by an exponential function whose exponent is the sum of the weights of the satisfied formulas, while the latter simply adds up the penalties of the satisfied formulas. On the other hand, they are monotonically increasing/decreasing as more formulas are satisfied.

In view of Theorem~2 from~\cite{lee16weighted}, which tells us how to embed Markov Logic into $\lpmln$, it follows from Theorem~\ref{thm:lpmln2wc} that  MAP inference in MLN can also be reduced to the optimal stable model finding of programs with weak constraints. 
For any Markov Logic Network $\Pi$, let ${\rm mln2wc}(\Pi)$ be the union of ${\rm lpmln2wc}(\Pi)$ plus choice rules~$\{A\}^{\rm ch}$ for all atoms~$A$. 

\begin{thm}\label{thm:mln2asp}
For any Markov Logic Network $\Pi$, the most probable models of $\Pi$ are precisely the optimal stable models of the  program with weak constraints ${\rm mln2wc}(\Pi)$.
\end{thm}

Similarly, MAP inference in ProbLog and Pearl's Causal Models can be reduced to finding an optimal stable model of a program with weak constraints in view of the reduction of ProbLog to $\lpmln$ (Theorem~4 from \cite{lee16weighted}) and the reduction of Causal Models to $\lpmln$ (Theorem~4 from \cite{lee15markov}) thereby allowing us to apply combinatorial optimization methods in standard ASP solvers to these languages.

\subsection{Alternative Translations}

Instead of aggregating the weights of satisfied formulas, we may aggregate the weights of formulas that are not satisfied. 
Let ${\rm lpmln2wc^{pnt}}(\Pi)$ be a modification of ${\rm lpmln2wc}(\Pi)$ by replacing \eqref{hard-wc} with 
\[
  :\sim\ \ \neg F\ \ [1 @ 1] 
\] 
and \eqref{soft-wc} with 
 \[
   :\sim\ \ \neg F\ \ [w@0].
 \]

Intuitively, when $F$ is not satisfied, the stable model gets the penalty $1$ at level $1$, or $w$ at level $0$ depending on whether $F$ is a hard or soft formula.

\begin{cor}\label{cor:lpmln2wc-pnt}
Theorem~\ref{thm:lpmln2wc} remains true when ${\rm lpmln2wc}(\Pi)$ is replaced by ${\rm lpmln2wc^{pnt}}(\Pi)$.
\end{cor}

This alternative view of assigning weights to stable models, in fact, originates from Probabilistic Soft Logic (PSL) \cite{bach15hinge}, where the probability density function of an interpretation is obtained from the sum over the ``penalty'' from the formulas that are ``distant'' from satisfaction. 
This view will lead to a slight advantage when we further turn the translation into the input language of ASP solvers (See Footnote~6).

The current ASP solvers do not allow arbitrary formulas to appear in weak constraints. For computation using the ASP solvers, let  ${\rm lpmln2wc^{pnt,rule}}(\Pi)$ be the translation by turning each weighted formula  $w_i: F_i$ in $\Pi$ into
\[
\ba {rcl}
  \neg F_i &\rar& \i{unsat}(i)  \\
   \neg \i{unsat}(i) & \rar& F_i \\
   & :\sim & \i{unsat}(i)\ \ \ \ \  [w_i @ l].
\ea
\]
where $\i{unsat}(i)$ is a new atom, and $l=1$ if $w_i$ is $\alpha$ and $l=0$ otherwise.

\begin{cor} \label{cor:lpmln2wc-pnt-rule}
Let  $\Pi$ be an $\lpmln$ program. There is a 1-1 correspondence $\phi$ between the most probable stable models of $\Pi$ and the optimal stable models of ${\rm lpmln2wc^{pnt,rule}}(\Pi)$, where $\phi(I) = I\cup\{  \i{unsat}(i) \mid w_i: F_i \in \Pi, \ I\not\models F_i \}$.
\end{cor}

The corollary allows us to compute the most probable stable models (MAP estimates) of the $\lpmln$ program using the combination of {\sc f2lp}~\footnote{\url{http://reasoning.eas.asu.edu/f2lp/}} and {\sc clingo} \footnote{\url{http://potassco.sourceforge.net}} (assuming that the weights are approximated to integers).  System {\sc f2lp} turns this program with formulas into the input language of {\sc clingo}, so {\sc clingo} can be used to compute the theory.

If the unweighted $\lpmln$ program is already in the rule form $\i{Head}\ar\i{Body}$
that is allowed in the input languages of  {\sc clingo} and {\sc dlv}, we may avoid the use of {\sc f2lp} by slightly modifying the translation ${\rm lpmln2wc^{pnt,rule}}$ by turning each weighted rule
\[
  w_i: \i{Head}_i\ar\i{Body}_i
\] 
instead into
\[
\ba {rcl}
  \i{unsat}(i) &\ar& \i{Body}_i, \no\ \i{Head}_i  \\
  \i{Head}_i &\ar& \i{Body}_i, \no\ \i{unsat}(i) \\
  &:\sim&  \i{unsat}(i)\ \ \ \ [w_i@l]
\ea
\]
where $l=1$ if $w_i$ is $\alpha$ and $l=0$ otherwise.

In the case when $\i{Head}_i$ is $\bot$, the translation can be further simplified: we simply turn $w_i : \bot \ar \i{Body}_i$ into  $:\sim  \i{Body}_i \ \ \ \ [w_i@l]$.\footnote{Alternatively, we may turn it into the ``reward'' way, i.e., turning it into $:\sim \no\ \i{Body}_i [-w_i]$, but the rule may not be in the input language of {\sc clingo}.}

\medskip\noindent
{\bf Example~\ref{ex:1} continued:}\ \ 
{\sl 
For program \eqref{ex1}, the simplified translation ${\rm lpmln2wc^{pnt,rule}}$ yields 
\[
\footnotesize
\ba {l|l|lr}
  \i{unsat}(1) \ar p,\no\ q  &  q \ar p, \no\ \i{unsat}(1) &  :\sim \i{unsat}(1) & [10@0] \\
  \i{unsat}(2) \ar p, \no\ r   &  r \ar p, \no\ \i{unsat}(2)  &  :\sim \i{unsat}(2) &  [1@0] \\
  \i{unsat}(3) \ar \no\ p     &    p \ar \no\ \i{unsat}(3) &  :\sim \i{unsat}(3) & [5@0]  \\
                                   &   &  :\sim \no\ r & [-20@0] 
\ea
\]
}

\section{Turning P-log into $\lpmln$} \label{sec:plog2lpmln}

\subsection{Review: P-log} \label{sec:syntax}

\subsubsection{Syntax}
 
A {\em sort} is a set of symbols.
A {\em constant} $c$ maps an element in the {\em domain} $s_1\times\cdots\times s_n$ to an element in the {\em range} $s_0$ (denoted by $\i{Range}(c)$), where each of $s_0,\dots,s_n$ is a sort. 
A {\em sorted propositional signature} is a special case of propositional signatures constructed from a set of constants and their associated sorts, consisting of all propositional atoms $c(\vec{u})\mvis v$ where
$c: s_1\times\cdots\times s_n\rar s_0$, and $\vec{u}\in s_1\times\cdots\times s_n$, and $v\in s_0$.\footnote{%
Note that here ``='' is just a part of the symbol for propositional atoms, and is not  equality in first-order logic. }
Symbol $c(\vec{u})$ is called an {\em attribute} and $v$ is called its {\em value}.
If the range $s_0$ of $c$ is $\{\false,\true\}$ then $c$ is called {\em Boolean}, and $c(\vec{u})\mvis\true$ can be abbreviated as $c(\vec{u})$ and $c(\vec{u})\mvis\false$ as~$\sneg c(\vec{u})$. 

The signature of a P-log program is the union of two propositional signatures $\sigma_1$ and $\sigma_2$, where $\sigma_1$ is a sorted propositional signature, and $\sigma_2$ is a usual propositional signature consisting of atoms $\i{Do}(c(\vec{u})\mvis v)$, $\i{Obs}(c(\vec{u})\mvis v)$ and $\i{Obs}(c(\vec{u})\!\ne\! v)$ for all atoms $c(\vec{u})\mvis v$ in $\sigma_1$.

A P-log program $\Pi$ of signature $\sigma_1\cup\sigma_2$ is a tuple
\begin{equation}\label{plog}
\Pi = \langle {\bf R}, {\bf S}, {\bf P}, {\bf Obs}, {\bf Act} \rangle
\end{equation}
where the signature of each of ${\bf R}$, ${\bf S}$, and ${\bf P}$ is $\sigma_1$ and the signature of each of ${\bf Obs}$ and ${\bf Act}$ is $\sigma_2$ such that
\begin{itemize}
\item  ${\bf R}$ is a set of {\em normal rules} of the form
\[ 
  A \ar B_1, \dots, B_m, \no\ B_{m+1}, \dots, \no\ B_n
\] 
where $A, B_1,\dots, B_n$ are atoms ($0\leq m\leq n$).

\item ${\bf S}$ is a set of {\em random selection rules} of the form
\beq
\ba l
[r]\ \ random(c(\vec{u}): \{x: p(x)\})\ar \i{Body}
\ea 
\eeq{random}
where $r$ is a unique identifier,  $p$ is a boolean constant with a unary argument, and $\i{Body}$ is a set of literals. 
{$x$ is a schematic variable ranging over the argument sort of $p$.} Rule~\eqref{random} is called a {\em random selection} rule for $c(\vec{u})$.
Intuitively, rule \eqref{random} says that if $\i{Body}$ is true, the value of $c(\vec{u})$ is selected at random from the set $\i{Range}(c)\cap \{x: p(x)\} $ unless this value is fixed by a deliberate action, i.e., $Do(c(\vec{u}) \mvis  v)$ for some value $v$.  

\item ${\bf P}$ is a set of so-called {\em probability atoms (pr-atoms)} 
of the form 
\beq 
   pr_r(c(\vec{u})\mvis v \mid C) = p
\eeq{pratom}
where $r$ is the identifier of some random selection rule for $c(\vec{u})$ in~${\bf S}$; $c(\vec{u})\mvis v \in\sigma_1$; $C$ is a set of literals; and $p$ is a real number in $[0, 1]$. We say pr-atom \eqref{pratom} is {\em associated} with the random selection rule whose identifier is $r$.

\item ${\bf Obs}$ is a set of atomic facts for representing ``observation'': 
$Obs(c(\vec{u}) \mvis  v)$ and $Obs(c(\vec{u}) \neq v)$.

\item ${\bf Act}$ is a set of atomic facts for representing a deliberate action: 
$ 
Do(c(\vec{u}) \mvis  v).
$ 
\end{itemize}

\subsubsection{Semantics} \label{sub:semantics}

Let $\Pi$ be a P-log program \eqref{plog} of signature $\sigma_1 \cup \sigma_2$. The possible worlds of $\Pi$, denoted by $\omega(\Pi)$, are the stable models of $\tau(\Pi)$, a (standard) ASP program with the propositional signature 
\[ 
   \sigma_1\cup\sigma_2\cup\{\i{Intervene}(c(\vec{u})) \mid c(\vec{u}) \text{ is an attribute occurring in } {\bf S}\}
\]
that accounts for the logical part of P-log.
Due to lack of space we refer the reader to~\cite{baral09probabilistic} for the definition of $\tau(\Pi)$.

An atom $c(\vec{u}) \mvis  v$ is called {\em possible} in a possible world $W$ due to a random selection rule \eqref{random} if $\Pi$ contains \eqref{random} such that $W \models\i{Body} \land p(v)\land \neg\i{Intervene}(c(\vec{u}))$.\footnote{%
Note that this is slightly different from the original definition of P-log from~\cite{baral09probabilistic}, according to which, if $\i{Intervene}(c(\vec{u}))$ is true, the probability of $c(\vec{u}) \mvis  v$ is determined by the default probability, which is a bit unintuitive.} 
Pr-atom \eqref{pratom} is {\em applied} in $W$ if $c(\vec{u}) \mvis  v$ is possible in $W$ due to $r$ and $W \models C$.

As in \cite{baral09probabilistic}, we assume that all P-log programs $\Pi$ satisfy the following conditions:
\begin{itemize}
\item
\textbf{Condition 1 [Unique random selection rule]: }
If a P-log program $\Pi$ contains two random selection rules for $c(\vec{u})$:
\[ 
\ba l
   \left[r_1\right]\ random(c(\vec{u}): \{x: p_1(x) \}) \ar \i{Body}_1, \\[0.3em]
   \left[r_2\right]\ random(c(\vec{u}): \{x: p_2(x) \}) \ar \i{Body}_2,
\ea
\]
then no possible world of $\Pi$ satisfies both $\i{Body}_1$ and $\i{Body}_2$. 

\item
\textbf{Condition 2 [Unique probability assignment]: }
If a P-log program $\Pi$ contains a random selection rule for $c(\vec{u})$:
\[ 
  \left[ r \right] \ random(c(\vec{u}): \{ x: p(x) \}) \ar \i{Body}
\]
along with two different pr-atoms:    
\[ 
\ba l
   pr_r(c(\vec{u}) \mvis  v \ | \ C_1) = p_1 , \\[0.3em]
   pr_r(c(\vec{u}) \mvis  v \ | \ C_2) = p_2 ,
\ea
\]
then no possible world of $\Pi$ satisfies $\i{Body}$, $C_1$, and $C_2$ together. 
\end{itemize}

\smallskip
Given a P-log program $\Pi$, a possible world $W\in\omega(\Pi)$, and an atom $c(\vec{u}) \mvis  v$ possible in $W$, by \textbf{Condition 1}, it follows that there is exactly one random selection rule \eqref{random} such that 
$W\models\i{Body}$. Let $r_{W, c(\vec{u})}$ denote this random selection rule, and let $AV_W(c(\vec{u})) = \{v' \mid $ there exists a pr-atom 
$pr_{r_{W, c(\vec{u})}}(c(\vec{u}) = v'\mid C)=p$ that is applied in $W$ for some $C$ and $p \}$.
We then define the following notations:
\begin{itemize}
\item  If $v\in AV_W(c(\vec{u}))$, there exists a pr-atom 
\hbox{$pr_{r_{W, c(\vec{u})}} (c(\vec{u}) \mvis  v\mid C) = p $} in $\Pi$ 
for some $C$ and $p$ such that $W\models C$. By \textbf{Condition 2}, 
for any other $pr_{r_{W, c(\vec{u})}} (c(\vec{u}) \mvis  v\mid C') = p'$ in $\Pi$, it follows that $W \not \models C'$. So there is only one pr-atom that is applied in $W$ for $c(\vec{u})\mvis v$, and we define
\[ 
   \i{PossWithAssPr}(W, c(\vec{u}) \mvis  v) = p.
\] 
{(``$c(\vec{u}) \mvis  v$ is possible in $W$ with assigned probability $p$.'')}

\item 
If $v\not\in AV_W(c(\vec{u}))$, we define
\[ 
   \i{PossWithDefPr}(W, c(\vec{u})\mvis v) = \max\big( p, 0 \big),
\] 
where $p$ is 
{\small
\begin{multline*}
   \frac{1 - \sum_{v'\in AV_W(c(\vec{u}))} \i{PossWithAssPr}(W, c(\vec{u}) \mvis  v')}
{ 
  |\{v'' \mid c(\vec{u})\mvis v'' \text{ is possible in } W \text{ and } v'' \not\in AV_W(c(\vec{u})) \}|
}. 
\tagthis\label{pwdp}
\end{multline*}
}

{ ( ``$c(\vec{u}) \mvis  v$ is possible in $W$ with the default probability.'') }

The $\max$ function is used to ensure that the default probability is nonnegative.~\footnote{In~\cite{baral09probabilistic}, a stronger condition of  ``unitariness" is imposed to prevent \eqref{pwdp} from being negative.} 
\end{itemize}

For each possible world $W\in \omega (\Pi)$, and each atom $c(\vec{u}) \mvis  v$ possible in $W$, the probability of $c(\vec{u}) \mvis  v$ to {\em happen} in $W$ is defined as:
\[
\ba l
P(W, c(\vec{u}) \mvis  v) = \\ \hspace{2em}
\begin{cases}
\i{PossWithAssPr}(W, c(\vec{u}) \mvis  v) & \text{if $v\in AV_W(c(\vec{u}))$;} \\
\i{PossWithDefPr}(W, c(\vec{u}) \mvis  v) & \text{otherwise.}
\end{cases}
\ea
\]

The {\em unnormalized probability} of a possible world $W$ is defined as
\[ 
  \hat{\mu}_{\Pi}(W) = \prod_{c(\vec{u})\mvis v\in W \text{ and }\atop c(\vec{u})\mvis v \text{ is possible in } W} P(W, c(\vec{u}) \mvis  v),
\]
and, assuming $\Pi$ has at least one possible world with nonzero unnormalized probability, the \emph{normalized probability} of $W$ is defined as
\[ 
  \mu_{\Pi}(W) = \frac{\hat{\mu}_{\Pi}(W)}{\sum_{W_i\in\omega(\Pi)}\hat{\mu}_{\Pi}(W_i)}.
\]

We say $\Pi$ is {\em consistent} if $\Pi$ has at least one possible world {with a non-zero probability}.

\begin{example} \label{ex:monty}
Consider a variant of the Monty Hall Problem encoding in P-log from~\cite{baral09probabilistic}  to illustrate the probabilistic nonmonotonicity in the presence of assigned probabilities.  There are four doors, behind which are three goats and one car. The guest picks  door~1, and Monty, the show host who always opens one of the doors with a goat, opens door~2. Further, while the guest and Monty are unaware, 
the statistics is that in the past, with 30\% chance the prize was behind door 1, and with 20\% chance, the prize was behind  door~3. Is it still better to switch to another door? 
This example can be formalized in P-log program $\Pi$, using both assigned probability and default probability, as 
{\small
\[ 
\ba{l}
\sneg \i{CanOpen}(d) \ar \i{Selected} \mvis  d.   \ \ \ \ (d \in \{ 1, 2, 3, 4 \}) , \\
\sneg \i{CanOpen}(d) \ar \i{Prize} \mvis  d. \\
\i{CanOpen}(d) \ar \no\ \sneg \i{CanOpen}(d). \\

random(\i{Prize}). \ \ \ \ \ \ random(\i{Selected}). \\
random(\i{Open}: \{x: \i{CanOpen}(x)\}).\\
pr(\i{Prize}\mvis 1) = 0.3. \ \ pr(\i{Prize}\mvis 3)=0.2. \\ 
Obs(\i{Selected}\mvis 1). \ \   Obs(\i{Open}\mvis 2).  \ \ Obs(\i{Prize} \neq 2).  \\
\ea
\] 
}

The possible worlds of $\Pi$ are as follows:
{\footnotesize
\bi
\ii  $W_1 = \{   \i{Obs}(\i{Selected}=1), \i{Obs}(\i{Open}=2), \i{Obs}(\i{Prize}\ne 2), 
                          \i{Selected}\mvis 1, \i{Open}\mvis 2, \i{Prize}\mvis 1,                        
                      \i{CanOpen}(1)=\false, \i{CanOpen}(2)=\true, \i{CanOpen}(3)=\true, \i{CanOpen}(4)=\true \}$ 
\ii $W_2 = \{   \i{Obs}(\i{Selected}=1), \i{Obs}(\i{Open}=2), \i{Obs}(\i{Prize}\ne 2), 
                         \i{Selected}\mvis 1, \i{Open}\mvis 2, \i{Prize}\mvis 3,
                      \i{CanOpen}(1)=\false, \i{CanOpen}(2)=\true, \i{CanOpen}(3)=\false, \i{CanOpen}(4)=\true \}$

\ii $W_3 = \{  \i{Obs}(\i{Selected}=1), \i{Obs}(\i{Open}=2), \i{Obs}(\i{Prize}\ne 2),  
                         \i{Selected}\mvis 1, \i{Open}\mvis 2, \i{Prize}\mvis 4, 
                      \i{CanOpen}(1)=\false, \i{CanOpen}(2)=\true, \i{CanOpen}(3)=\true, \i{CanOpen}(4)=\false \}$.
\ei
}
The probability of each atom to happen is 
\[
\small
\ba l
P(W_i, \i{Selected}\mvis 1) = \i{PossWithDefPr}(W, \i{Selected}\mvis 1)  = 1/4 \\ 
\\
P(W_1, \i{Open}\mvis 2) = \i{PossWithDefPr}(W_1, \i{Open}\mvis 2) = 1/3  \\
P(W_2, \i{Open}\mvis 2) = \i{PossWithDefPr}(W_2, \i{Open}\mvis 2) = 1/2  \\
P(W_3, \i{Open}\mvis 2) = \i{PossWithDefPr}(W_3, \i{Open}\mvis 2) = 1/2  \\ \\
P(W_1, \i{Prize}\mvis 1) = \i{PossWithAssPr}(W_1, \i{Prize}\mvis 1)  = 0.3\\
P(W_2, \i{Prize}\mvis 3) = \i{PossWithAssPr}(W_2, \i{Prize}\mvis 3)  = 0.2\\
P(W_3, \i{Prize}\mvis 4) = \i{PossWithDefPr}(W_3, \i{Prize}\mvis 4)  = 0.25 
\ea
\]
So,
\bi
{\small 
\ii $\hat{\mu}_\Pi(W_1) = 1/4 \times 1/3 \times 0.3 = 1/40$
\ii $\hat{\mu}_\Pi(W_2) = 1/4 \times 1/2 \times 0.2 = 1/40$
\ii $\hat{\mu}_\Pi(W_3) = 1/4 \times 1/2 \times 0.25 = 1/32$.
}
\ei
Thus, in comparison with staying ($W_1$), switching to door~$3$ ($W_2$) does not affect the chance, but switching to door $4$ ($W_3$) increases the chance by 25\%.  
\end{example}

\subsection{Turning P-log into $\lpmln$} \label{ssec:plog2lpmln}

We define translation ${\rm plog2lpmln}(\Pi)$ that turns a P-log program $\Pi$ into an $\lpmln$ program in a modular way.
First, every rule $R$ in $\tau(\Pi)$ (that is used in defining the possible worlds in P-log) is turned into a hard rule $\alpha: R$ in ${\rm plog2lpmln}(\Pi)$. 
In addition, ${\rm plog2lpmln}(\Pi)$ contains the following rules to associate probability to each possible world of $\Pi$. Below $x$, $y$ denote schematic variables, and $W$ is a possible world of~$\Pi$.

\bigskip
\noindent
{\bf Possible Atoms:}\ \ \ 
For each random selection rule \eqref{random} for $c(\vec{u})$ in ${\bf S}$
and for each $v\in\i{Range}(c)$,  ${\rm plog2lpmln}(\Pi)$ includes
\beq
   \i{Poss}_r(c(\vec{u})=v) \ar \i{Body}, p(v), \no\ \i{Intervene}(c(\vec{u}))
\eeq {poss}
Rule~\eqref{poss} expresses that $c(\vec{u})=v$ is possible in $W$ due to $r$ if $W\vDash \i{Body}\land p(v)\land \neg \i{Intervene}(c(\vec{u}))$.

\smallskip\noindent
{\bf Assigned Probability:}\ \ \ 
For each pr-atom~\eqref{pratom}
in \textbf{P}, ${\rm plog2lpmln}(\Pi)$ contains the following  rules:
{\small
\begin{align}
\alpha:&\ \ \i{PossWithAssPr}_{r, C}(c(\vec{u}) \mvis  v) \ar \notag\\
  & \hspace{5.5em}  \i{Poss}_r(c(\vec{u})=v),  C  \taglabel{pratomapplied} \\
\alpha:&\ \ \i{AssPr}_{r, C}(c(\vec{u}) \mvis  v) \ar  \notag \\
  &  \hspace{5.5em} c(\vec{u}) \mvis  v, \i{PossWithAssPr}_{r, C}(c(\vec{u}) \mvis  v)    \taglabel{asspr} \\               
ln(p):& \ \  \bot \ar \no\ \i{AssPr}_{r, C}(c(\vec{u}) \mvis  v)  \hspace{5mm}  \text{ ($p >0$) }  \taglabel{w-asspr}   \\ 
\alpha:&\ \ \bot\ar \i{AssPr}_{r,C}(c(\vec{u})\mvis v)   \hspace{11mm}\text{ ($p = 0$)} \notag \\ 
\alpha:& \ \ \i{PossWithAssPr}(c(\vec{u}) \mvis  v) \ar  \i{PossWithAssPr}_{r, C}(c(\vec{u}) \mvis  v).  \notag 
\end{align}
}
Rule~\eqref{pratomapplied} expresses the condition under which pr-atom \eqref{pratom} is applied in a possible world $W$. Further,  if $c(\vec{u}) \mvis  v$ is true in $W$ as well, rules~\eqref{asspr}  and \eqref{w-asspr}  contribute the assigned probability $e^{ln(p)} = p$ to the unnormalized probability of $W$ as a factor when $p>0$.  

\medskip\noindent
{\bf Denominator for Default Probability:}\ \ \ 
For each random selection rule \eqref{random} for $c(\vec{u})$ in ${\bf S}$
and for each $v\in\i{Range}(c)$,  ${\rm plog2lpmln}(\Pi)$ includes
{\small
\beq
\ba l
   \alpha: \ \i{PossWithDefPr}(c(\vec{u}) \mvis  v) \ar  \\
   \hspace{1cm}                \i{Poss}_r(c(\vec{u})\mvis v), \no\ \i{PossWithAssPr}(c(\vec{u}) \mvis  v)
\ea
\eeq{posswithdefpr} 
\beq    
\ba {rl} 
   \alpha:& \ \i{NumDefPr}(c(\vec{u}), x) \ar   \\ 
   & \hspace{1em}        c(\vec{u}) \mvis  v,   \i{PossWithDefPr}(c(\vec{u})=v),  \\            
   & \hspace{1em}         x = \#count\{y: \i{PossWithDefPr}(c(\vec{u}) \mvis  y) \}  \\
\ea
\eeq{numdefpr}             
\beq
\ba {rl}
  ln(\frac{1}{m}):& \bot \ar \no\ \i{NumDefPr}(c(\vec{u}), m)\\ 
  & \hspace{8em}  (m = 2, \dots, |\i{Range}(c) |)
\ea
\eeq{w-numdefpr}  
}
Rule \eqref{posswithdefpr} asserts that $c(\vec{u}) \mvis  v$ is possible in $W$ with a default probability if it is possible in $W$ and not possible with an assigned probability.
Rule~\eqref{numdefpr} expresses, intuitively, that $\i{NumDefPr}(c(\vec{u}), x)$ is true if there are exactly $x$ different values $v$ such that $c(\vec{u}) \mvis  v$ is possible in $W$ with a default probability, and there is at least one of them that is also true in $W$.  This value $x$ is the denominator of \eqref{pwdp}. 
Then rule~\eqref{w-numdefpr} contributes the factor $1/x$ to the unnormalized probability of $W$ as a factor. 

\medskip\noindent
{\bf Numerator for Default Probability:}\ \ \  
\begin{itemize}
\item Consider each random selection rule 
$[r]\ \ random(c(\vec{u}): \{x: p(x)\})\ar \i{Body}$
for $c(\vec{u})$ in \textbf{S} along with all pr-atoms associated with it in \textbf{P}:
\[ 
\ba l
    pr_r(c(\vec{u}) \mvis  v_1  \mid  C_1) = p_1   \\
    \dots  \\ 
    pr_r(c(\vec{u}) \mvis  v_n  \mid  C_n) = p_n 
\ea
\]

\noindent where $n\ge 1$, and $v_i$ and $v_j$ $(i\ne j)$ may be equal.
For each $v\in \i{Range}(c)$, ${\rm plog2lpmln(\Pi)}$ contains the following rules:\footnote{The sum aggregate can be represented by ground first-order formulas under the stable model semantics under the assumption that the Herbrand Universe is finite \cite{ferraris11logic}. In the general case, it can be represented by generalized quantifiers~\cite{lee12stable1} or infinitary propositional formulas~\cite{harrison14thesemantics}. In the input language of ASP solvers, which does not allow real number arguments, $p_i$ can be approximated to integers of some fixed interval.}
{\small 
\begin{align*}
 \alpha: &\ \   \i{RemPr}(c(\vec{u}), 1\!-\!y) \ar  \i{Body} \\ 
              &\ \ \hspace{0.5em} c(\vec{u}) \mvis  v, \i{PossWithDefPr}(c(\vec{u}) \mvis  v), \\
              &\ \  \hspace{0.5em} y = \#{\rm sum}\{p_1:\i{PossWithAssPr}_{r, C_1}(c(\vec{u}) \mvis  v_1); \\  
              & \ \  \hspace{3.5em}  \dots; p_n:\i{PossWithAssPr}_{r, C_n}(c(\vec{u}) \mvis  v_n)\}.   
              \taglabel{rempr}   \\
\alpha: &\ \  \i{TotalDefPr}(c(\vec{u}), x) \ar  \i{RemPr}(c(\vec{u}), x), x > 0   \taglabel{totaldefpr} \\
ln(x): & \ \   \bot\ar \no\ \i{TotalDefPr}(c(\vec{u}), x)    \taglabel{w-totaldefpr} \\
\alpha: &\ \  \bot\ar \i{RemPr}(c(\vec{u}), x), x \le 0.   \taglabel{TotalDefPr1}
\end{align*}
}
\vspace{-4mm}

\noindent
In rule~\eqref{rempr}, $y$ is the sum of all assigned probabilities.  
Rules~\eqref{totaldefpr} and \eqref{w-totaldefpr} are to account for the numerator of \eqref{pwdp} when $n>0$. The variable $x$ stands for the numerator of \eqref{pwdp}. 
Rule~\eqref{TotalDefPr1} is to avoid assigning a non-positive default probability to a possible world.
\end{itemize}

Note that most rules in ${\rm plog2lpmln(\Pi)}$ are hard rules. The soft rules \eqref{w-asspr}, \eqref{w-numdefpr}, \eqref{w-totaldefpr} cannot be simplified as atomic facts, e.g., 
$ln(\frac{1}{m}): \  \i{NumDefPr}(c(\vec{u}), m)$ in place of \eqref{w-numdefpr}, which is in contrast with the use of probabilistic choice atoms in the distribution semantics based probabilistic logic programming language, such as ProbLog. 
This is related to the fact that the probability of each atom to happen in a possible word in P-log is derived from assigned and default probabilities, and not from independent probabilistic choices like the other probabilistic logic programming languages. In conjunction with the embedding of ProbLog in $\lpmln$ \cite{lee16weighted}, it is interesting to note that both kinds of probabilities can be captured in $\lpmln$ using different kinds of rules. 

\medskip
\noindent{\bf Example~\ref{ex:monty} Continued}\ \ 
For the program $\Pi$ in Example~\ref{ex:monty}, ${\rm plog2lpmln}(\Pi)$ consists of the rules $\alpha: R$ for each rule $R$ in $\tau(\Pi)$ and the following rules.

\bigskip
\noindent
{\bf Possible Atoms:}\ \ 
{\small
\[
\ba l
\alpha : \i{Poss}(\i{Prize} =d) \leftarrow \no\ \i{Intervene}(\i{Prize})\\

\alpha : \i{Poss}(\i{Selected} =d) \leftarrow \no\ \i{Intervene}(\i{Selected})\\

\alpha : \i{Poss}(\i{Open} =d) \leftarrow \i{CanOpen}(d), \no\ \i{Intervene}(\i{Open})\\
\ea 
\]
}

\noindent
{\bf Assigned Probability:}\ \ \ 
{\small
\[
\ba l
\alpha : \i{PossWithAssPr}(\i{Prize} = 1) \leftarrow \i{Poss}(\i{Prize} =1) \\

\alpha : \i{AssPr}(\i{Prize} = 1) \leftarrow \i{Prize} = 1, \i{PossWithAssPr}(\i{Prize} = 1) \\

ln(0.3): \bot \leftarrow \no\ \i{AssPr}(\i{Prize} = 1)\\
\\
\alpha : \i{PossWithAssPr}(\i{Prize} = 3) \leftarrow \i{Poss}(\i{Prize} =3) \\

\alpha : \i{AssPr}(\i{Prize} = 3) \leftarrow \i{Prize} = 3, \i{PossWithAssPr}(\i{Prize} = 3) \\

ln(0.2): \bot \leftarrow \no\ \i{AssPr}(\i{Prize} = 3)\\
\ea 
\]
}

(We simplified slightly not to distinguish $\i{PossWithAssPr}(\cdot)$ and $\i{PossWithAssPr}_{r,C}(\cdot)$  because there is only one random selection rule  for $\i{Prize}$ and both $pr$-atoms for $\i{Prize}$ has empty conditions. )

\medskip\noindent
{\bf Denominator for Default Probability:}\ \ \ 
{\small
\[
\ba l
\alpha : \i{PossWithDefPr}(\i{Prize} =d) \leftarrow \\
\hspace{3em} \i{Poss}(\i{Prize} =d), \no\ \i{PossWithAssPr}(\i{Prize} =d)\\

\alpha : \i{PossWithDefPr}(\i{Selected} =d) \leftarrow  \\
\hspace{3em} \i{Poss}(\i{Selected} =d), \no\ \i{PossWithAssPr}(\i{Selected} =d)\\

\alpha : \i{PossWithDefPr}(\i{Open} =d) \leftarrow  \\
\hspace{3em}\i{Poss}(\i{Open} =d), \no\ \i{PossWithAssPr}(\i{Open} =d)
\ea 
\]
}
{\small
\[
\ba l
\alpha : \i{NumDefPr}(\i{Prize},x) \leftarrow \\ 
\hspace{3em} \i{Prize} = d, \i{PossWithDefPr}(\i{Prize} =d),\\
\hspace{3em}  x = \# count \{ y: \i{PossWithDefPr}(\i{Prize} =y) \} \\

\alpha : \i{NumDefPr}(\i{Selected},x) \leftarrow \\
\hspace{3em} \i{Selected} = d, \i{PossWithDefPr}(\i{Selected} =d), \\
\hspace{3em} x = \# count \{ y: \i{PossWithDefPr}(\i{Selected} =y) \} \\

\alpha : \i{NumDefPr}(\i{Open},x) \leftarrow \\
\hspace{3em} \i{Open} = d, \i{PossWithDefPr}(\i{Open} =d), \\
\hspace{3em} x = \# count \{ y: \i{PossWithDefPr}(\i{Open} =y) \} \\

ln(\frac{1}{m}): \leftarrow \no\ \i{NumDefPr}(c,m) \\
\hspace{6em}    (c\in \{ \i{Prize}, \i{Selected}, \i{Open} \}, m\in \{ 2,3,4\}  )
\ea 
\]
}

\noindent
{\bf Numerator for Default Probability:}\ \ \ 
{\small
\[ 
\ba l
\alpha : \i{RemPr}(\i{Prize}, 1\!-\!x) \leftarrow \i{Prize} = d, \i{PossWithDefPr}(\i{Prize} = d), \\
\hspace{5em} x = \#{\rm sum}\{0.3:\i{PossWithAssPr}(\i{Prize}\mvis 1);  \\
\hspace{92pt} 0.2:\i{PossWithAssPr}(\i{Prize}\mvis 3) \} \\

\alpha : \i{TotalDefPr}(\i{Prize},x) \leftarrow \i{RemPr}(\i{Prize},x), x>0 \\

ln(x): \bot \leftarrow \no\ \i{TotalDefPr}(\i{Prize},x) \\

\alpha : \bot \leftarrow \i{RemDefPr}(\i{Prize},x), x\leq 0 
\ea
\]
}

Clearly, the signature of ${\rm plog2lpmln}(\Pi)$ is a superset of the signature of $\Pi$.  Further, ${\rm plog2lpmln}(\Pi)$ is linear-time constructible.
The following theorem tells us that there is a 1-1 correspondence between the set of the possible worlds with non-zero probabilities of $\Pi$ and the set of the stable models of ${\rm plog2lpmln}(\Pi)$ such that each stable model is an extension of the possible world, and the probability of each possible world of $\Pi$ coincides with the probability of the corresponding stable model of ${\rm plog2lpmln}(\Pi)$.
\begin{thm}\label{thm:plog2lpmln}
Let $\Pi$ be a consistent P-log program. There is a 1-1 correspondence $\phi$ between the set of the possible worlds of $\Pi$ with non-zero probabilities and the set of probabilistic stable models of ${\rm plog2lpmln}(\Pi)$ such that 
\bi
\ii  For every possible world $W$ of $\Pi$ that has a non-zero probability, $\phi(W)$ is a probabilistic stable model of ${\rm plog2lpmln(\Pi)}$, and $\mu_\Pi(W) = P_{\rm plog2lpmln(\Pi)}(\phi(W))$.

\ii For every probabilistic stable model $I$ of ${\rm plog2lpmln(\Pi)}$,  the restriction of $I$ onto the signature of $\tau(\Pi)$, denoted $I|_{\sigma(\tau(\Pi))}$, is a possible world of $\Pi$ and $\mu_\Pi(I|_{\sigma(\tau(\Pi))})>0$.
\ei
\end{thm}

\begin{proof} (Sketch) 
We can check that the following mapping $\phi$ is the 1-1 correspondence. 
\begin{enumerate}
\item  $\phi(W)\models \i{Poss}_r(c(\vec{u}) \mvis  v)$ iff $c(\vec{u})\mvis v$ is possible in $W$ due to $r$.

\item  For each pr-atom
    $pr_r(c(\vec{u}) \mvis  v \ | \ C) = p$ in $\Pi$, 
    
    $\phi(W)\models \i{PossWithAssPr}_{r, C}(c(\vec{u}) \mvis  v)$ iff 
    this pr-atom is applied in $W$.
    
\item  For each pr-atom $pr_r(c(\vec{u}) \mvis  v \ | \ C) = p$ in~$\Pi$, 
          
        $\phi(W)\models \i{AssPr}_{r, C}(c(\vec{u}) \mvis  v)$ iff this pr-atom is applied in~$W$, and $W\models c(\vec{u})\mvis v$.
        
\item  $\phi(W)\models \i{PossWithAssPr}(c(\vec{u}) \mvis  v)$ iff $v \in  AV_W(c(\vec{u}))$.

\item  $\phi(W)\models \i{PossWithDefPr}(c(\vec{u}) \mvis  v)$ iff $c(\vec{u})\mvis v$ is possible in $W$ and $v \not \in  AV_W(c(\vec{u}))$.

\item  $\phi(W)\models \i{NumDefPr}(c(\vec{u}), m)$ iff there exist exactly $m$ different values $v$ such that $c(\vec{u})\mvis v$ is possible in $W$; $v \not \in AV_W(c(\vec{u}))$; and, for one of such $v$, $W\models c(\vec{u})\mvis v$.

\item  $\phi(W)\models \i{RemPr}(c(\vec{u}), k)$ iff 
           there exists a value $v$ such that 
           $W\models c(\vec{u})\mvis v$; 
           $c(\vec{u}) \mvis  v$ is possible in $W$; 
           $v \not \in AV_W(c(\vec{u}))$; and  
           $$k = 1 - \sum\limits_{ v\in AV_W(c(\vec{u})) } \i{PossWithAssPr}(W, c(\vec{u}) \mvis  v).$$

\item  $\phi(W)\models \i{TotalDefPr}(c(\vec{u}), k)$ iff 
           $\phi(W)\models \i{RemPr}(c(\vec{u}),k)$ and $k>0$.

\end{enumerate}

To check that $\mu_\Pi(W) = P_{\rm plog2lpmln(\Pi)}(\phi(W))$, note first that 
the weight of $\phi(W)$ is computed by multiplying $e$ to the power of the weights of rules \eqref{w-asspr}, \eqref{w-numdefpr}, \eqref{w-totaldefpr} that are satisfied by $\phi(W)$. Let's call this product $TW$.

Consider a possible world $W$ with a non-zero probability of $\Pi$ and $c(\vec{u})\mvis v$ that is satisfied by $W$.

If $c(\vec{u})\mvis v$ is possible in $W$ and pr-atom $pr_r(c(\vec{u})\mvis v \mid C) = p$  is applied in $W$ (i.e., $v\in AV_W(c(\vec{u}))$), then  the assigned probability is applied: $P(W, c(\vec{u})\mvis v) = p$. On the other hand,  by condition~3, $\phi(W)\models \i{AssPr}_{r,C}(c(\vec{u})\mvis v)$, so that from \eqref{w-asspr}, the same $p$ is a factor of $TW$.

If $c(\vec{u})\mvis v$ is possible in $W$ and $v\not\in AV_W(c(\vec{u}))$, the default probability is applied: $P(W, c(\vec{u})\mvis v) = p$ is computed by \eqref{pwdp}.
By Condition 8, $\phi(W)\models \i{TotalDefPr}(c(\vec{u}),x)$ where $x = 1 - \sum\limits_{ v'\in AV_W(c(\vec{u})) } \i{PossWithAssPr}(W, c(\vec{u}) \mvis  v')$. Since $\phi(W)\models \eqref{w-totaldefpr}$, it is a factor of $TW$, which is the same as the numerator of \eqref{pwdp}.  
Furthermore, by Condition~6,  $\phi(W)\models\i{NumDefPr}(c(\vec{u}),m)$, where 
$m$ is the 
denominator of \eqref{pwdp}. Since $\phi(W)\models \eqref{w-numdefpr}$, 
$\frac{1}{m}$ is a factor of $TW$. 
\qed
\end{proof}


\smallskip
\noindent{\bf Example~\ref{ex:monty} Continued}\ \ 
{\sl 
For the P-log program $\Pi$ for the Monty Hall problem, $\Pi' = {\rm plog2lpmln}(\Pi)$ has three probabilistic stable models $I_1$, $I_2$, and $I_3$, each of which is an extension of $W_1$, $W_2$, and  $W_3$ respectively, and satisfies the following atoms: 
$\i{Poss}(\i{Prize}\mvis i)$ for $i=1,2,3,4$;
$\i{Poss}(\i{Selected}\mvis i)$ for $i=1,2,3,4$;
$\i{PossWithAssPr}(\i{Prize}\mvis i)$ for $i=1,3$;
$\i{PossWithDefPr}(\i{Prize}\mvis i)$ for $i =2, 4$;
$\i{PossWithDefPr}(\i{Selected}\mvis i)$ for $i=1,2,3,4$;
$\i{NumDefPr}(\i{Selected}, 4)$.
\BOC
\begin{itemize}\addtolength{\itemsep}{-0.5mm}
\item 
$\i{Poss}(\i{Prize}\mvis i)$ for $i=1,2,3,4$;
\item 
$\i{Poss}(\i{Selected}\mvis i)$ for $i=1,2,3,4$;
\item 
$\i{PossWithAssPr}(\i{Prize}\mvis i)$ for $i=1,3$;
\item  
$\i{PossWithDefPr}(\i{Prize}\mvis i)$ for $i =2, 4$;
\item 
$\i{PossWithDefPr}(\i{Selected}\mvis i)$ for $i=1,2,3,4$;
\item 
$\i{NumDefPr}(\i{Selected}, 4)$.
\end{itemize}
\EOC
In addition, 
\bi
\ii $I_1\models \{  \i{AssPr}(\i{Prize}\mvis 1), \i{Poss}(\i{Open}\mvis 2), \\ 
 \i{Poss}(\i{Open}\mvis 3), \i{Poss}(\i{Open}\mvis 4), \\ 
 \i{PossWithDefPr}(\i{Open}\mvis 2), \i{PossWithDefPr}(\i{Open}\mvis 3),  \\ 
 \i{PossWithDefPr}(\i{Open}\mvis 4), \i{NumDefPr}(\i{Open},3)           \}$

\ii $I_2\models\{ \i{AssPr}(\i{Prize}\mvis 3), \i{Poss}(\i{Open}\mvis 2), \\
 \i{Poss}(\i{Open}\mvis 4),\i{PossWithDefPr}(\i{Open}\mvis 2), \\ 
 \i{PossWithDefPr}(\i{Open}\mvis 4), \i{NumDefPr}(\i{Open},2)     \}$

\ii $I_3\models \{ 
\i{Poss}(\i{Open}\mvis 2),  \i{Poss}(\i{Open}\mvis 3), \\
 \i{PossWithDefPr}(\i{Open}\mvis 2), \i{PossWithDefPr}(\i{Open}\mvis 3), \\
 \i{NumDefPr}(\i{Open},2),  \i{NumDefPr}(\i{Prize},2),  \\ 
  \i{RemPr}(\i{Prize},0.5), \i{TotalDefPr}(\i{Prize},0.5)
  \}$.
\ei

The unnormalized weight $W_{\Pi'}(I_i)$ of each probabilistic stable model $I_i$ is shown below. 
$w(\i{AssPr}_{r,C}(c(\vec{u})\mvis v))$ denotes the exponentiated weight of rule~\eqref{w-asspr}; 
$w(\i{NumDefPr}(c(\vec{u}),m))$ denotes the exponentiated weight of rule~\eqref{w-numdefpr};
$w(\i{TotalDefPr}(c(\vec{u}),x))$ denotes the exponentiated weight of rule~\eqref{w-totaldefpr}.
\bi
\ii 
$W_{\Pi'}(I_1) = w(\i{NumDefPr}(\i{Selected},4)) \times w(\i{AssPr}(\i{Prize}\mvis 1)) 
                     \times w(\i{NumDefPr}(Open,3))
                     =\frac{1}{4} \times \frac{3}{10} \times \frac{1}{3} = \frac{1}{40} $.
\ii
$W_{\Pi'}(I_2) = w(\i{NumDefPr}(\i{Selected},4)) \times w(\i{AssPr}(\i{Prize}\mvis 3)) 
                     \times w(\i{NumDefPr}(Open,2))
                     = \frac{1}{4}  \times \frac{2}{10} \times \frac{1}{2} = \frac{1}{40}$;

\ii 
$W_{\Pi'}(I_3) = w(\i{NumDefPr}(\i{Selected},4)) \times  
                     \times w(\i{NumDefPr}(Open,2))\times \times w(\i{NumDefPr}(\i{Prize},2) \times w(\i{TotalDefPr}(\i{Prize},0.5) 
                      = \frac{1}{4} \times \frac{1}{2} \times \frac{1}{2} \times\frac{5}{10}= \frac{1}{32}$. 

\ei
}

\smallskip 
Combining the translations ${\rm plog2lpmln}$ and ${\rm lpmln2wc}$, one can compute P-log MAP inference using standard ASP solvers.

\section{Conclusion}

In this paper, we show how $\lpmln$ is related to weak constraints and P-log. Weak constraints are a relatively simple extension to ASP programs, while P-log is highly structured but a more complex extension. $\lpmln$ is shown to be a good middle ground language that clarifies the relationships. 
We expect the relationships will help us to apply the mathematical and computational results developed for one language to another language. 

\bigskip\noindent
{\bf Acknowledgments}\ \ 
We are grateful to Evgenii Balai, Yi Wang and anonymous referees for their useful comments on the draft of this paper.
This work was partially supported by the National Science Foundation under Grants IIS-1319794 and IIS-1526301. 


\bibliographystyle{aaai}

\def\wdbprm{W^{\prime\prime}}
\def\pdbprm{P^{\prime\prime}}
\def\smdbprm{SM^{\prime\prime}}
\def\grdtms{\textbf{\textit{At}}}

\newcommand{\argmin}{\mathop{\mathrm{argmin}}\limits}
\newcommand{\argmax}{\mathop{\mathrm{argmax}}\limits}
\newcommand\numberthis{\addtocounter{equation}{1}\tag{\theequation}}

\onecolumn
{\large \textbf{ Appendix to ``$\lpmln$, Weak Constraints, and P-log''}}

The appendix contains 
\begin{itemize}
\item
Proofs in order of 
{\bf Proposition \ref{prop:lpmln2wc}},
{\bf Theorem \ref{thm:lpmln2wc}},
{\bf Theorem \ref{thm:mln2asp}},
{\bf Corollary \ref{cor:lpmln2wc-pnt}},
{\bf Corollary \ref{cor:lpmln2wc-pnt-rule}},
{\bf Corollary \ref{cor:lpmln2wc-pnt-asp}},
{\bf Corollary \ref{cor:lpmln2wc-pnt-clingo-s}},
and {\bf Theorem~\ref{thm:plog2lpmln}};
({\bf Corollary \ref{cor:lpmln2wc-pnt-asp}} and {\bf Corollary \ref{cor:lpmln2wc-pnt-clingo-s}} are corollaries of {\bf Corollary \ref{cor:lpmln2wc-pnt-rule}} when ${\rm lpmln2wc^{pnt,rule}}$ is simplified)
\item
The full $\lpmln$ encoding and full ASP with weak constraints encoding of the variant Monty Hall problem.

\end{itemize}

\section{Proof of $\textbf{Proposition \ref{prop:lpmln2wc}}$} \label{sec:proofwc:prop}

\noindent{\bf Proposition~\ref{prop:lpmln2wc} \optional{prop:lpmln2wc}}\ \ 
{\sl
For any $\lpmln$ program $\Pi$, the set $\sm[\Pi]$ is exactly the set of the stable models of ${\rm lpmln2wc(\Pi)}$.
}
\medskip

\begin{proof}
To prove {\bf Proposition \ref{prop:lpmln2wc}}, it is sufficient to prove
\beq
\text{$I \in \sm[\Pi]$ iff $I$ is a stable model of ${\rm lpmln2wc}(\Pi)$.}
\eeq {t1_sm}
Since $I \in \sm[\Pi]$ iff $I$ is a stable model of $\o{\Pi_I}$, by definition, (\ref{t1_sm}) is equivalent to saying
\[
\text{$I$ is a minimal model of $\bigwedge\limits_{w: F \in \Pi , I\vDash F} F^I$ 
iff 
$I$ is a minimal model of $\bigwedge\limits_{w: F \in \Pi} (\{F\}^{\rm{ch}})^I$},
\]
which is true because
\[
\bigwedge\limits_{w: F \in \Pi} (\{F\}^{\rm{ch}})^I = \bigwedge\limits_{w: F \in \Pi, I\vDash F} (\{F\}^{\rm{ch}})^I \land \bigwedge\limits_{w: F \in \Pi, I\not \vDash F} (\{F\}^{\rm{ch}})^I = \bigwedge\limits_{w: F \in \Pi, I\vDash F} F^I.
\]
\end{proof}
\qed

\section{Proof of $\textbf{Theorem \ref{thm:lpmln2wc}}$} \label{sec:proofwc}
Let $\Pi$ be an $\lpmln$ program. 
By $\Pi^{\rm soft}$ we denote the set of all soft rules in $\Pi$, by $\Pi^{\rm hard}$ we denote the set of all hard rules in $\Pi$. 
For any $I \in \sm[\Pi]$, let $W^{\rm hard}_{\Pi}(I)=exp\bigg(\sum\limits_{w:F\;\in\; (\Pi^{\rm hard})_I} w\bigg)$ and $W^{\rm soft}_{\Pi}(I) = exp\bigg(\sum\limits_{w:F\;\in\; (\Pi^{\rm soft})_I} w\bigg)$, then
$I$ is a most probable stable model of $\Pi$ iff
\[
I \in \argmax_{J:~J \in \argmax_{K:~K \in {\rm SM}[\Pi]} W^{\rm hard}_{\Pi}(K)} W^{\rm soft}_{\Pi}(J).
\]
Let $\Pi'$ be an ASP program with weak constraints such that $Level \in \{ 0,1 \}$ for all weak constraints
\[
{\tt :\sim}\ F \ \ \  [\i{Weight}\ @\ \i{Level}]
\]
in $\Pi'$. $I$ is an optimal stable model of $\Pi'$ iff
\[
I \in \argmin_{J:~J \in \argmin_{K:~K \text{ is a stable model of } \Pi'} \mathit{Penalty}_{\Pi'}(K,1)} \mathit{Penalty}_{\Pi'}(J,0).
\]

\medskip
\noindent{\bf Theorem~\ref{thm:lpmln2wc} \optional{thm:lpmln2wc}}\ \ 
{\sl
For any $\lpmln$ program~$\Pi$, the most probable stable models of $\Pi$ are precisely the optimal stable models of the program with weak constraints ${\rm lpmln2wc}(\Pi)$.
}
\medskip

\begin{proof}
Let $\Pi'$ denote ${\rm lpmln2wc}(\Pi)$. To prove {\bf Theorem~\ref{thm:lpmln2wc}}, it is sufficient to prove
\[
\text{$I$ is a most probable stable model of $\Pi$ iff $I$ is an optimal stable model of $\Pi'$}
\]
which 
is equivalent to proving
\[
I \in \argmax_{J:~J \in \argmax_{K:~K \in {\rm SM}[\Pi]} W^{\rm hard}_{\Pi}(K)} W^{\rm soft}_{\Pi}(J) \text{ iff } I \in \argmin_{J:~J \in \argmin_{K:~K \text{ is a stable model of } \Pi'} \mathit{Penalty}_{\Pi'}(K,1)} \mathit{Penalty}_{\Pi'}(J,0).
\]
This is clear because 
\[
\ba {l l}
& \argmax_{J:~J \in \argmax_{K:~K \in {\rm SM}[\Pi]} W^{\rm hard}_{\Pi}(K)} W^{\rm soft}_{\Pi}(J) \\

=& \text{(by (\ref{t1_sm}) and the definition of $W^{\rm hard}_{\Pi}(I)$ and $W^{\rm soft}_{\Pi}(I)$)}  \\

& \argmax_{J:~J \in \argmax_{K:~K \text{ is a stable model of } \Pi'} exp\big(\sum\limits_{\alpha:F\;\in\; (\Pi^{\rm hard})_K} \alpha\big)} exp\Big(\sum\limits_{w:F\;\in\; (\Pi^{\rm soft})_J} w\Big)\\

=& \\

& \argmax_{J:~J \in \argmax_{K:~K \text{ is a stable model of } \Pi'} exp\big(\sum\limits_{\alpha:F\;\in\; \Pi^{\rm hard}, K\vDash F} 1\big)} exp\Big(\sum\limits_{w:F\;\in\; \Pi^{\rm soft}, J\vDash F} w\Big)\\

=& \\

& \argmin_{J:~J \in \argmin_{K:~K \text{ is a stable model of } \Pi'} \big(\sum\limits_{\alpha:F\;\in\; \Pi^{\rm hard}, K\vDash F} -1\big)} \Big(\sum\limits_{w:F\;\in\; \Pi^{\rm soft}, J\vDash F} -w\Big)\\

=& \\

& \argmin_{J:~J \in \argmin_{K:~K \text{ is a stable model of } \Pi'} \big(\sum\limits_{:\sneg~ F\left[ -1@1 \right] \in \Pi', K\vDash F} -1\big)} \Big(\sum\limits_{:\sneg~ F\left[ -w@0 \right] \in \Pi', J\vDash F} -w\Big)\\

=& \\

& \argmin_{J:~J \in \argmin_{K:~K \text{ is a stable model of } \Pi'} \mathit{Penalty}_{\Pi'}(K,1)} \mathit{Penalty}_{\Pi'}(J,0).
\ea
\]

\end{proof}
\qed

\section{Proof of $\textbf{Theorem \ref{thm:mln2asp}}$} \label{sec:proofasp}
\noindent{\bf Theorem~\ref{thm:mln2asp} \optional{thm:mln2asp}}\ \ 
{\sl
For any Markov Logic Network $\Pi$, the most probable models of $\Pi$ are precisely the optimal stable models of the  program with weak constraints ${\rm mln2wc}(\Pi)$.
}
\medskip

\begin{proof}
For any Markov Logic Network $\Pi$, we obtain an $\lpmln$ program $\Pi'$ from $\Pi$ by adding 
\[
\alpha: \{ A\}^{\rm ch}
\]
for every atom $A$ in $\Pi$.
By Theorem 2 in \cite{lee16weighted}, $\Pi$ and $\Pi'$ have the same probability distribution over all interpretations. Then for any interpretation $I$ of $\Pi$,
\begin{itemize}
\item
$I$ is a most probable model of the MLN program $\Pi$
\end{itemize}
iff
\begin{itemize}
\item
$I$ is a most probable stable model of the $\lpmln$ program $\Pi'$
\end{itemize}
iff (by \textbf{Theorem \ref{thm:lpmln2wc}})
\begin{itemize}
\item
$I$ is an optimal stable model of the ASP program with weak constraints ${\rm lpmln2wc}(\Pi')$
\end{itemize}
iff (since a choice rule is always satisfied, omiting the weak constraint ``$:\sneg \{A\}^{\rm ch}. \left[ -1@1 \right]$'' for all atoms $A$ in $\Pi$ doesn't affect what is an optimal stable model of ${\rm lpmln2wc}(\Pi')$)
\begin{itemize}
\item
$I$ is an optimal stable model of the ASP program with weak constraints ${\rm lpmln2wc}(\Pi) \cup \{ \{ \{ A \}^{\rm ch}\}^{\rm ch} \mid A \text{ is an atom in }\Pi\}$
\end{itemize}
iff (since for any interpretation $I$, the reduct of $\{ \{ A \}^{\rm ch}\}^{\rm ch}$ relative to $I$ is equivalent to the reduct of $\{ A \}^{\rm ch}$ relative to $I$, ${\rm lpmln2wc}(\Pi) \cup \{ \{ \{ A \}^{\rm ch}\}^{\rm ch} \mid A \text{ is an atom in }\Pi\}$ is strongly equivalent to ${\rm lpmln2wc}(\Pi) \cup \{ \{ A \}^{\rm ch} \mid A \text{ is an atom in }\Pi\}$)
\begin{itemize}
\item
$I$ is an optimal stable model of the ASP program with weak constraints ${\rm lpmln2wc}(\Pi) \cup \{ \{ A \}^{\rm ch} \mid A \text{ is an atom in }\Pi\}$
\end{itemize}
Thus we proved $I$ is a most probable model of an MLN program $\Pi$ iff $I$ is an optimal stable model of the ASP program with weak constraints ${\rm mln2wc}(\Pi)$.
\end{proof}

\qed

\section{Proof of $\textbf{Corollary \ref{cor:lpmln2wc-pnt}}$} \label{sec:proofwc-pnt}
\noindent{\bf Corollary~\ref{cor:lpmln2wc-pnt} \optional{cor:lpmln2wc-pnt}}\ \ 
{\sl
For any $\lpmln$ program~$\Pi$, the most probable stable models of $\Pi$ are precisely the optimal stable models of the program with weak constraints ${\rm lpmln2wc^{pnt}}(\Pi)$.
}
\medskip

\begin{proof}
Let $\Pi'$ denote ${\rm lpmln2wc^{pnt}}(\Pi)$. From (\ref{t1_sm}), it's clear that
\beq
\text{$I \in \sm[\Pi]$ iff $I$ is a stable model of $\Pi'$.}
\eeq {c1_sm}
To prove 
\[
\text{$I$ is a most probable stable model of $\Pi$ iff $I$ is an optimal stable model of ${\rm lpmln2wc^{pnt}}(\Pi)$},
\]
it is equivalent to proving
\[
I \in \argmax_{J:~J \in \argmax_{K:~K \in {\rm SM}[\Pi]} W^{\rm hard}_{\Pi}(K)} W^{\rm soft}_{\Pi}(J) \text{ iff } I \in \argmin_{J:~J \in \argmin_{K:~K \text{ is a stable model of } \Pi'} \mathit{Penalty}_{\Pi'}(K,1)} \mathit{Penalty}_{\Pi'}(J,0).
\]
This is clear because
\[
\ba {l l}
& \argmax_{J:~J \in \argmax_{K:~K \in {\rm SM}[\Pi]} W^{\rm hard}_{\Pi}(K)} W^{\rm soft}_{\Pi}(J) \\

=& \text{(by \eqref{c1_sm} and the definition of $W^{\rm hard}_{\Pi}(I)$ and $W^{\rm soft}_{\Pi}(I)$)} \\

& \argmax_{J:~J \in \argmax_{K:~K \text{ is a stable model of } \Pi'} exp\big(\sum\limits_{\alpha:F\;\in\; (\Pi^{\rm hard})_K} \alpha\big)} exp\Big(\sum\limits_{w:F\;\in\; (\Pi^{\rm soft})_J} w\Big)\\

=& \\

& \argmax_{J:~J \in \argmax_{K:~K \text{ is a stable model of } \Pi'} exp\big(\sum\limits_{\alpha:F\;\in\; \Pi^{\rm hard}, K \vDash F} 1\big)} exp\Big(\sum\limits_{w:F\;\in\; \Pi^{\rm soft}, J \vDash F} w\Big)\\

=& \text{(since 
$\sum\limits_{\alpha:F\;\in\; \Pi^{\rm hard}, K \vDash F} 1 + \sum\limits_{\alpha:F\;\in\; \Pi^{\rm hard}, K \not\vDash F} 1$ is a fixed integer that equals to the number of hard rules in $\Pi$,} \\

& ~~\text{and 
$\sum\limits_{w: F \in \Pi^{\rm soft}, J\vDash F} w + \sum\limits_{w: F \in \Pi^{\rm soft}, J\not \vDash F} w$ is a fixed real number that equals to the sum of the weights of all soft rules in $\Pi$)}  \\

& \argmin_{J:~J \in \argmin_{K:~K \text{ is a stable model of } \Pi'} \big(\sum\limits_{\alpha:F\;\in\; \Pi^{\rm hard}, K\not \vDash F} 1\big)} \Big(\sum\limits_{w:F\;\in\; \Pi^{\rm soft}, J\not \vDash F} w\Big)\\

=& \\

& \argmin_{J:~J \in \argmin_{K:~K \text{ is a stable model of } \Pi'} \big(\sum\limits_{:\sneg\ \neg F\left[ 1@1 \right] \in \Pi', K\not \vDash F} 1\big)} \Big(\sum\limits_{:\sneg\ \neg F\left[ w@0 \right] \in \Pi', J\not \vDash F} w\Big)\\

=& \\

& \argmin_{J:~J \in \argmin_{K:~K \text{ is a stable model of } \Pi'} \mathit{Penalty}_{\Pi'}(K,1)} \mathit{Penalty}_{\Pi'}(J,0).
\ea
\]

\end{proof}
\qed

\section{Proof of $\textbf{Corollary \ref{cor:lpmln2wc-pnt-rule}}$} \label{sec:proofwc-pnt-rule}
Let $\sigma$ and $\sigma'$ be signatures such that $\sigma \subseteq \sigma'$. For any two interpretations $I$, $J$ of the same signature $\sigma'$, we write $I <^{\sigma} J$ iff 
\begin{itemize}
\item
$I|_{\sigma} \subsetneq J|_{\sigma}$, and 
\item
$I$ and $J$ agree on $\sigma' \setminus \sigma$.
\end{itemize}

The proof of {\bf Corollary \ref{cor:lpmln2wc-pnt-rule}} will use a restricted version of Theorem 1 from \cite{bartholomew13onthestable}, which is reformulated as follows: 
\medskip
\begin{lemma}  \label{lem:sm-mm}
Let $F$ be a propositional formula. An interpretation $I$ is a stable model of $F$ relative to signature $\sigma$ iff
\begin{itemize}
\item
$I \vDash F^I$, 
\item
and no interpretation $J$ such that $J <^{\sigma} I$ satisfies $F^I$.
\end{itemize}
\end{lemma}

The proof of {\bf Corollary \ref{cor:lpmln2wc-pnt-rule}} will use a restricted version of the splitting theorem from \cite{ferr09b}, which is reformulated as follows:

\medskip
\noindent{\bf Splitting Theorem}\ \ 
{\sl
Let $\Pi_1$, $\Pi_2$ be two finite ground programs,
${\bf p}$, ${\bf q}$ be disjoint tuples of distinct atoms. If
\begin{itemize}
\item each strongly connected component of the dependency graph of $\Pi_1\cup\Pi_2$ w.r.t. ${\bf p}\cup{\bf q}$ is a subset of ${\bf p}$ or a subset of ${\bf q}$,
\item no atom in ${\bf p}$ has a strictly positive occurrence in $\Pi_2$, and
\item no atom in ${\bf q}$ has a strictly positive occurrence in $\Pi_1$,
\end{itemize}
then an interpretation $I$ of $\Pi_1\cup \Pi_2$ is a stable model of $\Pi_1\cup \Pi_2$ relative to ${\bf p}\cup{\bf q}$ if and only if $I$ is a stable model of $\Pi_1$ relative to ${\bf p}$ and $I$ is a stable model of $\Pi_2$ relative to ${\bf q}$.
}

\medskip
The proof of {\bf Corollary \ref{cor:lpmln2wc-pnt-rule}} will also use the following lemma. Here and after, $w_i: F_i$ denotes the $i$-th rule in $\Pi$, where $w_i$ could be $\alpha$ or a real number. 
\medskip
\begin{lemma}  \label{lem:c2}
Let  $\Pi$ be an $\lpmln$ program. There is a 1-1 correspondence $\phi$ between the set $\sm[\Pi]$ and the set of the stable models of ${\rm lpmln2wc^{pnt,rule}}(\Pi)$, where $\phi(I) = I\cup\{  \i{unsat}(i) \mid w_i: F_i \in \Pi, \ I\not\models F_i \}$.
\end{lemma}
\medskip
\begin{proof}
Let $\sigma$ be the signature of $\Pi$. We can check that the following mapping $\phi$ is a 1-1 correspondence:
\[
\phi(I) = I \cup \{ \i{unsat}(i) \mid w_i: F_i \in \Pi, I \not\models F_i \}
\]
where $\phi(I)$ is of an extended signature $\sigma \cup \{ \i{unsat}(i) \mid w_i: F_i \in \Pi \}$. 

To prove $\phi$ is a 1-1 correspondence between the set $\sm[\Pi]$ and the set of the stable models of
\beq
\bigwedge\limits_{w_i:\ F_i \in \Pi}  \Big( 
(F_i \leftarrow \neg \i{unsat}(i)) \land 
(\i{unsat}(i) \leftarrow \neg F_i) \Big),
\eeq{cor2_sm}
it is sufficient to prove the following two bullets.
\begin{itemize}
\item \textbf{prove: for every interpretation $I \in \sm[\Pi]$, $\phi(I)$ is a stable model of (\ref{cor2_sm}).}

Assume $I \in \sm[\Pi]$, by (\ref{t1_sm}), $I$ is a stable model of 
\[
\bigwedge\limits_{w_i:\ F_i \in \Pi} (F_i \leftarrow \neg \neg F_i).
\]
By {\bf Lemma \ref{lem:sm-mm}}, we know
\begin{itemize}
\item
$I \vDash$
\beq
\bigwedge\limits_{w_i:\ F_i \in \Pi, I \vDash F_i} \Big(F_i^I \Big),
\eeq {cor2:b}
\item
and no interpretation $K$ of signature $\sigma$ such that $K <^{\sigma} I$ satisfies (\ref{cor2:b}).
\end{itemize}

To prove $\phi(I)$ is a stable model of (\ref{cor2_sm}), by the splitting theorem, it is sufficient to show
\begin{itemize}
\item
$\phi(I)$ is a stable model of 
$
\bigwedge\limits_{w_i:\ F_i \in \Pi} \Big( \i{unsat}(i) \leftarrow \neg \ F_i \Big)
$
relative to $\{ \i{unsat}(i) \mid w_i: F_i \in \Pi \}$, and
\item
$\phi(I)$ is a stable model of
$
\bigwedge\limits_{w_i:\ F_i \in \Pi}  \Big( 
F_i \leftarrow \neg \i{unsat}(i) \Big)
$
relative to $\sigma$;
\end{itemize}
which is equivalent to showing
\begin{itemize}
\ii[{\bf (a)}]
$\phi(I) \vDash \bigwedge\limits_{w_i:\ F_i \in \Pi} \Big(\i{unsat}(i) \leftrightarrow \neg \ F_i \Big)$, 
\ii[{\bf (b.1)}]
	$\phi(I) \vDash$
	\beq
	\bigwedge\limits_{w_i:\ F_i \in \Pi, \phi(I) \vDash F_i}  \Big( F_i \leftarrow \neg \i{unsat}(i) \Big)^{\phi(I)} \land
	\bigwedge\limits_{w_i:\ F_i \in \Pi, \phi(I) \not \vDash F_i}  \Big( F_i \leftarrow \neg \i{unsat}(i) \Big)^{\phi(I)},
	\eeq {cor2_b1}
\ii[{\bf (b.2)}]
	and no interpretation $L$ of signature $\sigma \cup \{ \i{unsat}(i) \mid w_i: F_i \in \Pi \}$ such that $L <^{\sigma} \phi(I)$ satisfies (\ref{cor2_b1}).
\end{itemize}
It's clear that {\bf (a)} is true by the definition of $\phi(I)$. As for {\bf (b.1)}, since $\phi(I) \vDash 
( F_i \leftrightarrow \neg \i{unsat}(i))$ for all $w_i: F_i \in \Pi$, (\ref{cor2_b1}) is equivalent to
\[
	\bigwedge\limits_{w_i:\ F_i \in \Pi, \phi(I) \vDash F_i}  \Big( F_i^{\phi(I)} \leftarrow \top \Big) \land
	\bigwedge\limits_{w_i:\ F_i \in \Pi, \phi(I) \not \vDash F_i}  \Big( \bot \leftarrow \bot \Big).
\]
Then {\bf (b.1)} is equivalent to saying 
$\phi(I) \vDash$
\beq
	\bigwedge\limits_{w_i:\ F_i \in \Pi, \phi(I) \vDash F_i}  \Big( F_i^{\phi(I)} \Big),
\eeq {cor2_b2}
which is further equivalent to saying $I \vDash$ (\ref{cor2:b}). 
As for {\bf (b.2)}, assume for the sake of contradiction that there exists an interpretation $L$ such that $L <^{\sigma} \phi(I)$ satisfies (\ref{cor2_b1}). Since (\ref{cor2_b1}) is equivalent to (\ref{cor2_b2}), $L \vDash \bigwedge\limits_{w_i:\ F_i \in \Pi, \phi(I) \vDash F_i}  \Big( F_i^{\phi(I)} \Big)$. Thus we know $L|_{\sigma} <^{\sigma} I$ and $L|_{\sigma} \vDash$ (\ref{cor2:b}), which contradicts with ``there is no interpretation $K$ such that $K <^{\sigma} I$ satisfies (\ref{cor2:b})''. So both {\bf (b.1)} and {\bf (b.2)} are true. Consequently, $\phi(I)$ is a stable model of (\ref{cor2_sm}).

\item \textbf{prove: for every stable model $J$ of (\ref{cor2_sm}), $J|_{\sigma} \in \sm[\Pi]$ and $J = \phi(J|_{\sigma})$.}

Assume $J$ is a stable model of (\ref{cor2_sm}), by the splitting theorem, 
\begin{itemize}
\item
$J$ is a stable model of 
$
\bigwedge\limits_{w_i:\ F_i \in \Pi} \Big( \i{unsat}(i) \leftarrow \neg \ F_i \Big)
$
relative to $\{ \i{unsat}(i) \mid w_i: F_i \in \Pi \}$, and 
\item
$J$ is a stable model of
$
\bigwedge\limits_{w_i:\ F_i \in \Pi}  \Big( 
F_i \leftarrow \neg \i{unsat}(i) \Big)
$
relative to $\sigma$.
\end{itemize}
Thus we have
\begin{itemize}
\ii[{\bf (c)}]
$J \vDash \i{unsat}(i) \leftrightarrow \neg F_i$ for all $w_i: F_i \in \Pi$, which follows that $J = J|_{\sigma} \cup \{ \i{unsat}(i) \mid w_i: F_i \in \Pi, J|_{\sigma} \vDash \neg \ F_i \}$. In other words, $J = \phi(J|_{\sigma})$.
\ii[{\bf (d.1)}]
	Since $J\vDash (F_i \leftarrow \neg \i{unsat}(i))$ for all $w_i: F_i \in \Pi$, $J \vDash$ 
	\beq
	\bigwedge\limits_{w_i:\ F_i \in \Pi}  \Big( 
F_i^J \leftarrow (\neg \i{unsat}(i))^J \Big),
	\eeq {cor2_d1}
\ii[{\bf (d.2)}]
	and no interpretation $L$ of signature $\sigma \cup \{ \i{unsat}(i) \mid w_i: F_i \in \Pi \}$ such that $L <^{\sigma} J$ satisfies (\ref{cor2_d1}).
\end{itemize}
Since $J \vDash \i{unsat}(i) \leftrightarrow \neg F_i$ for all $w_i: F_i \in \Pi$, (\ref{cor2_d1}) is equivalent to
\[
	\bigwedge\limits_{w_i:\ F_i \in \Pi, J|_{\sigma} \vDash F_i}  \Big( F_i^J \leftarrow \top \Big) \land
	\bigwedge\limits_{w_i:\ F_i \in \Pi, J|_{\sigma} \not \vDash F_i}  \Big( \bot \leftarrow \bot \Big),
\]
which is further equivalent to
\beq
	\bigwedge\limits_{w_i:\ F_i \in \Pi, J|_{\sigma} \vDash F_i}  \Big( F_i^{J|_{\sigma}} \Big).
\eeq {cor2_d2}
Thus by {\bf (d.1)}, $J|_{\sigma} \vDash$ (\ref{cor2_d2}); and by {\bf (d.2)}, it's easy to show that no interpretation $K$ such that $K <^{\sigma} J|_{\sigma}$ satisfies ({\ref{cor2_d2}}). (Assume for the sake of contradiction, there exists an interpretation $K$ such that $K <^{\sigma} J|_{\sigma}$ satisfies (\ref{cor2_d2}). Let $L = K \cup \{ \i{unsat}(i) \mid w_i: F_i \in \Pi, J\vDash \neg \ F_i \}$. Then $L <^{\sigma} J$ and $L\vDash$ (\ref{cor2_d2}). Since (\ref{cor2_d2}) is equivalent to ({\ref{cor2_d1}}), $L \vDash$ ({\ref{cor2_d1}}), which contradicts with {\bf (d.2)}.)

Then by {\bf Lemma \ref{lem:sm-mm}}, $J|_{\sigma}$ is a stable model of 
$
\bigwedge\limits_{w_i:\ F_i \in \Pi} (F_i \leftarrow \neg \neg F_i).
$
By (\ref{t1_sm}), $J|_{\sigma} \in \sm[\Pi]$.

\end{itemize}
\end{proof}
\qed

\noindent{\bf Corollary~\ref{cor:lpmln2wc-pnt-rule} \optional{cor:lpmln2wc-pnt-rule}}\ \ 
{\sl
Let  $\Pi$ be an $\lpmln$ program. There is a 1-1 correspondence $\phi$ between the most probable stable models of $\Pi$ and the optimal stable models of ${\rm lpmln2wc^{pnt,rule}}(\Pi)$, where $\phi(I) = I\cup\{  \i{unsat}(i) \mid w_i: F_i \in \Pi, \ I\not\models F_i \}$.
}
\medskip

\begin{proof}
Let $\sigma$ be the signature of $\Pi$. We can check that the following mapping $\phi$ is a 1-1 correspondence:
\[
\phi(I) = I \cup \{ \i{unsat}(i) \mid w_i: F_i \in \Pi, I\not\models \ F_i \}
\]
where $\phi(I)$ is of an extended signature $\sigma \cup \{ \i{unsat}(i) \mid w_i: F_i \in \Pi \}$.
By {\bf Lemma \ref{lem:c2}}, we know $\phi$ is a 1-1 correspondence between the set $\sm[\Pi]$ and the set of the stable models of ${\rm lpmln2wc^{pnt,rule}}(\Pi)$. Let $\Pi'$ denote ${\rm lpmln2wc^{pnt,rule}}(\Pi)$, $J'$ and $K'$ denote interpretations of signature $\sigma \cup \{ \i{unsat}(i) \mid w_i: F_i \in \Pi \}$. To prove
\[
\text{$I$ is a most probable stable model of $\Pi$ iff $\phi(I)$ is an optimal stable model of ${\rm lpmln2wc^{pnt,rule}}(\Pi)$},
\]
 it is equivalent to proving
\[
I \in \argmax_{J:~J \in \argmax_{K:~K \in {\rm SM}[\Pi]} W^{\rm hard}_{\Pi}(K)} W^{\rm soft}_{\Pi}(J) \text{ iff } \phi(I) \in \argmin_{J':~J' \in \argmin_{K':~K' \text{ is a stable model of } \Pi'} \mathit{Penalty}_{\Pi'}(K',1)} \mathit{Penalty}_{\Pi'}(J',0),
\]
which is further equivalent to proving
\[
I \in \argmax_{J:~J \in \argmax_{K:~K \in {\rm SM}[\Pi]} W^{\rm hard}_{\Pi}(K)} W^{\rm soft}_{\Pi}(J) \text{ iff } I \in \argmin_{J:~J \in \argmin_{K:~K \in {\rm SM}[\Pi]} \mathit{Penalty}_{\Pi'}(\phi(K),1)} \mathit{Penalty}_{\Pi'}(\phi(J),0).
\]
This is clear because 
\[
\ba {l l}
& \argmax_{J:~J \in \argmax_{K:~K \in {\rm SM}[\Pi]} W^{\rm hard}_{\Pi}(K)} W^{\rm soft}_{\Pi}(J) \\

=& \text{(by the definition of $W^{\rm hard}_{\Pi}(I)$ and $W^{\rm soft}_{\Pi}(I)$)} \\

&\argmax_{J:~J \in \argmax_{K:~K \in {\rm SM}[\Pi]} exp\big(\sum\limits_{\alpha:F_i\;\in\; (\Pi^{\rm hard})_K} \alpha\big)} exp\Big(\sum\limits_{w_i:F_i\;\in\; (\Pi^{\rm soft})_J} w_i\Big)\\

=&\\

& \argmax_{J:~J \in \argmax_{K:~K \in {\rm SM}[\Pi]} \big(\sum\limits_{\alpha:F_i\;\in\; \Pi^{\rm hard}, K \vDash F_i} \alpha \big)} \Big(\sum\limits_{w_i:F_i\;\in\; \Pi^{\rm soft}, J \vDash F_i} w_i\Big)\\

=& \\ 


& \argmin_{J:~J \in \argmin_{K:~K \in {\rm SM}[\Pi]} \big(\sum\limits_{\alpha:F_i\;\in\; \Pi^{\rm hard}, K\not \vDash F_i} \alpha \big)} \Big(\sum\limits_{w_i:F_i\;\in\; \Pi^{\rm soft}, J\not \vDash F_i} w_i\Big)\\

=& \text{(since for any interpretation $K \in \sm[\Pi]$, $\phi(K)\vDash \i{unsat(i)}$ iff $K \not \vDash F_i$)} \\

& \argmin_{J:~J \in \argmin_{K:~K \in {\rm SM}[\Pi]} \big(\sum\limits_{:\sneg\ unsat(i)\left[ \alpha @1 \right] \in \Pi', \phi(K) \vDash unsat(i)} \alpha \big)} \Big(\sum\limits_{:\sneg\ unsat(i)\left[ w_i@0 \right] \in \Pi', \phi(J) \vDash unsat(i)} w_i\Big)\\

=& \\

& \argmin_{J:~J \in \argmin_{K:~K \in {\rm SM}[\Pi]} \mathit{Penalty}_{\Pi'}(\phi(K),1)} \mathit{Penalty}_{\Pi'}(\phi(J),0).
\ea
\]

\end{proof}
\qed

\section{Proof of $\textbf{Corollary \ref{cor:lpmln2wc-pnt-asp}}$} \label{sec:proofwc-pnt-asp}
For any $\lpmln$ program $\Pi$ such that all unweighted rules of $\Pi$ are in the rule form $\i{Head}\ar\i{Body}$, let ${\rm lpmln2wc^{pnt,clingo}}(\Pi)$ be the translation by turning each weighted rule $w_i : \i{Head}_i\ar\i{Body}_i$ in $\Pi$ into
\[
\ba {rcl}
\i{unsat}(i) &\ar& \i{Body}_i, \no\ \i{Head}_i  \\
\i{Head}_i &\ar& \i{Body}_i, \no\ \i{unsat}(i) \\
&:\sim&  \i{unsat}(i)\ \ \ \ [w_i@l]
\ea
\]
where $l=1$ if $w_i$ is $\alpha$ and $l=0$ otherwise. 

\medskip
\begin{cor} \label{cor:lpmln2wc-pnt-asp} 
Let  $\Pi$ be an $\lpmln$ program such that all unweighted rules of $\Pi$ are in the rule form $\i{Head}\ar\i{Body}$.
There is a 1-1 correspondence $\phi$ between the most probable stable models of $\Pi$ and the optimal stable models of ${\rm lpmln2wc^{pnt,clingo}}(\Pi)$, where
$\phi(I) = I \cup \{ \i{unsat}(i) \mid w_i: \i{Head}_i \ar \i{Body}_i \in \Pi, I\vDash \i{Body}_i \land \neg \i{Head}_i \}$.
\end{cor}

\medskip
\begin{proof}
Let $\sigma$ denote the signature of $\Pi$. 
Since the weak constraints of ${\rm lpmln2wc^{pnt,clingo}}(\Pi)$ are exactly the same as the weak constraints of ${\rm lpmln2wc^{pnt,rule}}(\Pi)$, 
by {\bf Corollary \ref{cor:lpmln2wc-pnt-rule}}, it suffices to prove that for any interpretation $I$ of the signature $\sigma \cup \{ \i{unsat}(i) \mid w_i: \i{Head}_i \ar \i{Body}_i \in \Pi \}$,
\[
\text{$I$ is a stable model of ${\rm lpmln2wc^{pnt,clingo}}(\Pi)$ iff $I$ is a stable model of ${\rm lpmln2wc^{pnt,rule}}(\Pi)$.}
\]
By the splitting theorem, it is equivalent to proving
\begin{itemize}
\ii[{\bf (a)}]
$I$ is a stable model of 
$
\bigwedge\limits_{w_i:\ Head_i \ar Body_i \in \Pi} \Big( \i{unsat}(i) \ar \i{Body}_i \land \neg \i{Head}_i \Big)
$
relative to $\{ \i{unsat}(i) \mid w_i: \i{Head}_i \ar \i{Body}_i \in \Pi \}$, and 
\ii[{\bf (b)}]
$I$ is a stable model of
$
\bigwedge\limits_{w_i:\ Head_i \ar Body_i \in \Pi}  \Big( 
\i{Head}_i \ar \i{Body}_i \land \neg \i{unsat}(i) \Big)
$
relative to $\sigma$;
\end{itemize}
iff
\begin{itemize}
\ii[{\bf (c)}]
$I$ is a stable model of 
$
\bigwedge\limits_{w_i:\ Head_i \ar Body_i \in \Pi} \Big( \i{unsat}(i) \leftarrow \neg (\i{Head}_i \ar \i{Body}_i) \Big)
$
relative to $\{ \i{unsat}(i) \mid w_i: \i{Head}_i \ar \i{Body}_i \in \Pi \}$, and 
\ii[{\bf (d)}]
$I$ is a stable model of
$
\bigwedge\limits_{w_i:\ Head_i \ar Body_i \in \Pi}  \Big( 
(\i{Head}_i \ar \i{Body}_i) \leftarrow \neg \i{unsat}(i) \Big)
$
relative to $\sigma$.
\end{itemize}
This is clear because 
\begin{itemize}
\item
{\bf (a)} and {\bf (c)} are equivalent to saying
$
I \vDash \bigwedge\limits_{w_i:\ Head_i \ar Body_i \in \Pi} \Big( \i{unsat}(i) \leftrightarrow \i{Body}_i \land \neg \i{Head}_i \Big)
$ (by completion), and
\item
{\bf (b)} and {\bf (d)} are equivalent because
$\i{Head}_i \ar \i{Body}_i \land \neg \i{unsat}(i)$ is strongly equivalent to $(\i{Head}_i \ar \i{Body}_i) \leftarrow \neg \i{unsat}(i)$. It is because for any interpretation $J$,
\[
\ba {l l}
& \big( (\i{Head}_i \ar \i{Body}_i) \leftarrow \neg \i{unsat}(i) \big)^J \\

=& \\

& \begin{cases}
(\i{Head}_i \ar \i{Body}_i)^J \leftarrow (\neg \i{unsat}(i))^J & \text{if $J\vDash \i{Head}_i \lor \neg \i{Body}_i \lor \i{unsat}(i)$,} \\
\bot & \text{otherwise;}
\end{cases} \\

=& \\

& \begin{cases}
(\i{Head}_i^J \ar \i{Body}_i^J) \leftarrow (\neg \i{unsat}(i))^J & \text{if $J\vDash \i{Head}_i \lor \neg \i{Body}_i$,} \\
\bot \leftarrow \bot & \text{if $J \not \vDash \i{Head}_i \lor \neg \i{Body}_i$ and $J \vDash \i{unsat}(i)$,} \\
\bot & \text{otherwise;}
\end{cases} \\

\Leftrightarrow& \\

& \begin{cases}
\i{Head}_i^J \leftarrow \big(\i{Body}_i^J \land (\neg \i{unsat}(i))^J\big) & \text{if $J\vDash \i{Head}_i \lor \neg \i{Body}_i$,} \\
\i{Head}_i^J \leftarrow \big(\i{Body}_i^J \land (\neg \i{unsat}(i))^J\big) & \text{if $J \not \vDash \i{Head}_i \lor \neg \i{Body}_i$ and $J\vDash \i{unsat}(i)$,} \\
\bot & \text{otherwise;}
\end{cases} \\

=& \\

& \begin{cases}
\i{Head}_i^J \leftarrow \big(\i{Body}_i^J \land (\neg \i{unsat}(i))^J\big) & \text{if $J\vDash \i{Head}_i \lor \neg \i{Body}_i \lor \i{unsat}(i)$,} \\
\bot & \text{otherwise;}
\end{cases} \\

=& \\

& \begin{cases}
\i{Head}_i^J \leftarrow \big(\i{Body}_i \land (\neg \i{unsat}(i))\big)^J & \text{if $J\vDash \i{Head}_i \lor \neg \i{Body}_i \lor \i{unsat}(i)$,} \\
\bot & \text{otherwise;}
\end{cases} \\

=& \\

& \big( \i{Head}_i \ar \i{Body}_i \land \neg \i{unsat}(i) \big)^J.
\ea
\]
By Proposition 5 from \cite{ferraris11logic}, $\i{Head}_i \ar \i{Body}_i \land \neg \i{unsat}(i)$ is strongly equivalent to $(\i{Head}_i \ar \i{Body}_i) \leftarrow \neg \i{unsat}(i)$.
\end{itemize}
\end{proof}
\qed

\section{Proof of $\textbf{Corollary \ref{cor:lpmln2wc-pnt-clingo-s}}$} \label{sec:proofwc-pnt-clingo-s}
For any $\lpmln$ program $\Pi$ such that all unweighted rules of $\Pi$ are in the rule form $\i{Head}\ar\i{Body}$, let ${\rm lpmln2wc^{pnt,clingo}_{simp}}(\Pi)$ be the translation by turning each weighted rule $w_i : \i{Head}\ar\i{Body}$ in $\Pi$ into (where $l=1$ if $w_i$ is $\alpha$ and $l=0$ otherwise)
\[
\ba {rcl}
&:\sim&  \i{Body}_i\ \ \ \ [w_i@l]
\ea
\]
if $\i{Head}_i$ is $\bot$, or
\[
\ba {rcl}
\i{unsat}(i) &\ar& \i{Body}_i, \no\ \i{Head}_i  \\
\i{Head}_i &\ar& \i{Body}_i, \no\ \i{unsat}(i) \\
&:\sim&  \i{unsat}(i)\ \ \ \ [w_i@l]
\ea
\]
otherwise.

\medskip
\begin{cor} \label{cor:lpmln2wc-pnt-clingo-s} 
Let  $\Pi$ be an $\lpmln$ program such that all unweighted rules of $\Pi$ are in the rule form $\i{Head}\ar\i{Body}$.
There is a 1-1 correspondence $\phi$ between the most probable stable models of $\Pi$ and the optimal stable models of ${\rm lpmln2wc^{pnt,clingo}_{simp}}(\Pi)$, where
$\phi(I) = I \cup \{ \i{unsat}(i) \mid w_i: \i{Head}_i \ar \i{Body}_i \in \Pi, \i{Head}_i \text{ is not } \bot, I\vDash \i{Body}_i \land \neg \i{Head}_i\}$.
\end{cor}
\medskip

The proof of \textbf{Corollary \ref{cor:lpmln2wc-pnt-clingo-s}} will use the following lemma:
\medskip
\begin{lemma}\label{lem:constr} 
For any interpretation $I$ of an $\lpmln$ program $\Pi$, let $\Pi^{\rm constr}$ denote a set of weighted rules of the form $w: ~\leftarrow F$, where $w$ is $\alpha$ or a real number, $F$ is a first-order formula.
Then $I \in \sm[\Pi \cup \Pi^{\rm constr}]$ iff $I \in \sm[\Pi]$.
\end{lemma}
\medskip
\begin{proof}
\begin{itemize}
\item
$I \in \sm[\Pi \cup \Pi^{\rm constr}]$ 
\end{itemize}
iff (by definition)
\begin{itemize}
\item
$I$ is a stable model of 
$
\overline{\Pi_I}
\land
\bigwedge\limits_{\substack{
w:\ \bot \ar F \in \Pi^{\rm constr} \\ 
I \vDash \bot \ar F }}  \Big( \bot \ar F \Big)
$
\end{itemize}
iff (by theorem 3 in \cite{ferraris11stable})
\begin{itemize}
\item
$I$ is a stable model of $\overline{\Pi_I}$ 
and $I\vDash$
$
\bigwedge\limits_{\substack{
w:\ \bot \ar F \in \Pi^{\rm constr} \\ 
I \vDash \bot \ar F }}  \Big( \bot \ar F \Big)
$
\end{itemize}
iff (since 
$I\vDash
\bigwedge\limits_{\substack{
w:\ \bot \ar F \in \Pi^{\rm constr} \\ 
I \vDash \bot \ar F }}  \Big( \bot \ar F \Big)
$
is always true)
\begin{itemize}
\item
$I \in \sm[\Pi]$.
\end{itemize}

\end{proof}
\qed

\medskip

\noindent{\bf Proof of Corollary \ref{cor:lpmln2wc-pnt-clingo-s}}. ~~
We can check that the following mapping $\phi$ is a 1-1 correspondence:
\[
\phi(I) = I \cup \{ \i{unsat}(i) \mid w_i: \i{Head}_i \ar \i{Body}_i \in \Pi, \i{Head}_i \text{ is not } \bot, I\vDash \i{Body}_i \land \neg \i{Head}_i\},
\]
where $\phi(I)$ is of an extended signature $\sigma \cup \{ \i{unsat}(i) \mid w_i: \i{Head}_i \ar \i{Body}_i \in \Pi, \i{Head}_i \text{ is not } \bot \}$.

By {\bf Lemma \ref{lem:constr}}, we know $I \in \sm[\Pi]$ iff $I \in \sm[
\bigwedge\limits_{\substack{
w_i:\ Head_i \ar Body_i \in \Pi \\ 
Head_i \text{ is not } \bot }}  \Big( w_i:\ Head_i \ar Body_i \Big)]$. 

By {\bf Corollary \ref{cor:lpmln2wc-pnt-asp}}, we know $\phi$ is a 1-1 correspondence between the set 
$
 \sm[
\bigwedge\limits_{\substack{
w_i:\ Head_i \ar Body_i \in \Pi \\ 
Head_i \text{ is not } \bot }}  \Big( w_i:\ Head_i \ar Body_i \Big)]
$
and the set of the stable models of
\[ 
\bigwedge\limits_{\substack{w_i:\ Head_i \ar Body_i \in \Pi \\ \text{and $Head_i$ is not $\bot$}}}  \Big( 
(\i{unsat}(i) \ar \i{Body}_i \land \neg \i{Head}_i ) \land 
(\i{Head}_i \ar \i{Body}_i \land \neg \i{unsat}(i) ) \Big),
\] 
where $\phi(I) = I \cup \{ \i{unsat}(i) \mid w_i: \i{Head}_i \ar \i{Body}_i \in \Pi, \i{Head}_i \text{ is not } \bot, I\vDash \i{Body}_i \land \neg \i{Head}_i\}$.

Thus $\phi$ is a 1-1 correspondence between the set $\sm[\Pi]$ and the set of the stable models of ${\rm lpmln2wc^{pnt,clingo}_{simp}}(\Pi)$. 

Let $\Pi_{c4}$ denote ${\rm lpmln2wc^{pnt,clingo}_{simp}}(\Pi)$, $\Pi_{c3}$ denote ${\rm lpmln2wc^{pnt,clingo}}(\Pi)$, and $\phi_{c3}$ denote the 1-1 correspondence in {\bf Corollary \ref{cor:lpmln2wc-pnt-asp}}. By {\bf Corollary \ref{cor:lpmln2wc-pnt-asp}}, it is suffices to prove that for any interpretation $I\in \sm[\Pi]$ and $l\in \{ 0,1 \}$
\[
\i{Penalty}_{\Pi_{c4}}(\phi(I),l) = \i{Penalty}_{\Pi_{c3}}(\phi_{c3}(I),l).
\]
This is clear because
\[
\ba {l l}
& \i{Penalty}_{\Pi_{c4}}(\phi(I),0) \\

=&  \\

& 
\sum\limits_{\substack{:\sneg\ unsat(i)\left[ w_i@0 \right] \in \Pi_{c4} \text{ and } \phi(I) \vDash unsat(i)\\
\text{or } :\sneg Body_i \left[ w_i@0 \right] \in \Pi_{c4} \text{ and } \phi(I) \vDash Body_i
}} w_i
\\

=& \text{(since, by the definition of $\phi(I)$, when $\i{Head}_i$ is not $\bot$, $\phi(I)\vDash \i{unsat(i)}$ iff $I \vDash \i{Body}_i \land \neg \i{Head}_i$)}\\

& \sum\limits_{\substack{w_i:Head_i \ar Body_i\;\in\; \Pi^{\rm soft} \text{, $Head_i$ is not $\bot$, and } I \vDash Body_i \land \neg Head_i\\
\text{or } w_i:Head_i \ar Body_i\;\in\; \Pi^{\rm soft} \text{, $Head_i$ is $\bot$, and } I \vDash Body_i \land \neg Head_i
}} w_i
\\

=& \\

& 
\sum\limits_{w_i:Head_i \ar Body_i\;\in\; \Pi^{\rm soft} \text{ and } I \vDash Body_i \land \neg Head_i} w_i
\\

=& \text{(since $\phi_{c3}(I) \vDash unsat(i)$ iff $I \vDash Body_i \land \neg Head_i$)} \\

& 
\sum\limits_{:\sneg\ unsat(i)\left[ w_i@0 \right] \in \Pi_{c3} \text{ and } \phi_{c3}(I) \vDash unsat(i)} w_i
\\

=& \\

& \i{Penalty}_{\Pi_{c3}}(\phi_{c3}(I),0);
\ea
\]
and similarly,
\[
\i{Penalty}_{\Pi_{c4}}(\phi(I),1) = \i{Penalty}_{\Pi_{c3}}(\phi_{c3}(I),1).
\]
\qed

\newpage

\section{Proof of $\textbf{Theorem \ref{thm:plog2lpmln}}$} \label{sec:prooflpmln}

\subsection{Definition of $\tau(\Pi)$} \label{subsec:tauPi}
Given a P-log program $\Pi$ of the form~\eqref{plog} of signature $\sigma_1 \cup \sigma_2$, a (standard) ASP program $\tau(\Pi)$ with the propositional signature 
\[ \sigma_1\cup\sigma_2\cup\{\i{Intervene}(c(\vec{u})) \mid c(\vec{u}) \text{ is an attribute occurring in } {\bf S}\}, \]
where ${\bf S}$ is the set of random selection rules of $\Pi$, is constructed as follows: 

\begin{itemize}
 
\item $\tau(\Pi)$ contains all rules in {\bf R}.

\item For each attribute $c(\vec{u})$ in $\sigma_1$, for $v_1, v_2 \in \i{Range}(c)$, $\tau(\Pi)$ contains the following rule:
\[
\leftarrow c(\vec{u})= v_1, c(\vec{u})= v_2, v_1 \neq v_2
\]

\item For each random selection rule (\ref{random}) in ${\bf S}$ with $\i{Range}(c)=\left\{v_1,\dots, v_n\right\}$, $\tau(\Pi)$ contains the following rules:
\[
\ba l
c(\vec{u})= v_1; \dots; c(\vec{u})= v_n\ar Body, not\ \i{Intervene}(c(\vec{u})) \\
\ar c(\vec{u})= v, not\ p(v), Body, not\ \i{Intervene}(c(\vec{u}))
\ea
\]
where $\i{Intervene}(c(\vec{u}))$ means that the randomness of $c(\vec{u})$ is intervened (by an atomic fact $Do(c(\vec{u}) = v)$).

\item For each atomic fact $Obs(c(\vec{u}) = v)$ in ${\bf Obs}$,  $\tau(\Pi)$ contains the following rules:
\[
\ba l
Obs(c(\vec{u})=v) \\
\ar Obs(c(\vec{u})= v), \no\ c(\vec{u})= v
\ea
\]

\item For each atomic fact $Obs(c(\vec{u}) \neq v)$ in ${\bf Obs}$,  $\tau(\Pi)$ contains the following rules:
\[
\ba l
Obs(c(\vec{u})\neq v) \\
\ar Obs(c(\vec{u})\neq v), c(\vec{u})= v
\ea
\]

\item For each atomic fact $Do(c(\vec{u}) = v)$ in ${\bf Act}$,  $\tau(\Pi)$ contains the following rules:
\[ 
\ba l
Do(c(\vec{u})= v) \\
c(\vec{u})= v\ar Do(c(\vec{u})= v) \\
\i{Intervene}(c(\vec{u})) \ar Do(c(\vec{u})= v)
\ea
\] 
\end{itemize}

\subsection{Signature of ${\rm plog2lpmln}(\Pi)$} \label{subsec:signature}
For any real number $p \in \left[ 0,1 \right]$ and $b\in \{ \true, \false \}$, we define $\left[ p \right]^b$ as follows: $\left[ p \right]^b = p$ if $b = \true$, and $\left[ p \right]^b = 0$ if $b = \false$. Further, for any P-log program $\Pi$ and any $c(\vec{u})$ in {\bf S} of $\Pi$, we define the set of all possible remaining (unassigned) probabilities of $c(\vec{u})$ in $\Pi$, ${\bf p}_{rem}(c(\vec{u}), \Pi)$, as 
\[
\{ p\mid p = 1 - \sum\limits_{ p_i:~pr_r(c(\vec{u}) = v_i \ | \ C_i) = p_i \in \Pi} \left[ p_i \right]^{b_i}, \text{ where each } b_i\in\{ \true, \false \} \}.
\]

Given a P-log program $\Pi$ of the form~\eqref{plog} of signature $\sigma_1 \cup \sigma_2$, the signature of ${\rm plog2lpmln}(\Pi)$ is 
\[ \sigma_1\cup\sigma_2\cup\{\i{Intervene}(c(\vec{u})) \mid c(\vec{u}) \text{ is an attribute occurring in } {\bf S}\} \cup \sigma_3, \]
where $\sigma_3$ is a propositional signature constructed from $\Pi$ as follows:
\begin{align*}
\sigma_3 = & ~~~~
	\{\i{Poss}_r(c(\vec{u}) \mvis  v) \mid \text{$r$ is a random selection rule for $c(\vec{u})$ in $\Pi$ and $v\in \i{Range}(c) \}$} \\
	& \cup \{\i{\i{PossWithAssPr}}_{r, C}(c(\vec{u}) \mvis  v) \mid \text{there is a pr-atom 
	$pr_r(c(\vec{u}) = v \ | \ C) = p$ in $\Pi$} \} \\
	& \cup \{\i{AssPr}_{r, C}(c(\vec{u})\mvis  v) \mid \text{there is a pr-atom 
	$pr_r(c(\vec{u}) = v \ | \ C) = p$ in $\Pi$} \} \\
	& \cup \{\i{PossWithAssPr}(c(\vec{u}) \mvis  v) \mid \text{there is a random selection rule for $c(\vec{u})$ in $\Pi$ and $v\in \i{Range}(c) \}$} \\
	& \cup \{\i{PossWithDefPr}(c(\vec{u}) \mvis  v) \mid \text{there is a random selection rule for $c(\vec{u})$ in $\Pi$ and $v\in \i{Range}(c) \}$} \\
	& \cup \{\i{NumDefPr}(c(\vec{u}), m) \mid \text{there is a random selection rule for $c(\vec{u})$ in $\Pi$ and $m\in\{1, \dots, |\i{Range}(c)| \}$} \\
	& \cup \{\i{RemPr}(c(\vec{u}), k) \mid \text{there is a random selection rule for $c(\vec{u})$ in $\Pi$ and $k \in {\bf p}_{rem}(c(\vec{u}), \Pi)$}  \} \\
	& \cup \{\i{TotalDefPr}(c(\vec{u}), k) \mid \text{there is a random selection rule for $c(\vec{u})$ in $\Pi$, $k \in {\bf p}_{rem}(c(\vec{u}), \Pi)$, and $k>0$} \} .
\end{align*}

\medskip

Let $\sm '[\Pi]$ be the set
\[
\{ I \mid I \text{ is a stable model of $\overline{\Pi_I}$ that satisfy $\overline{\Pi^{\rm hard}}$} \}.
\]

The unnormalized weight of $I$ under $\Pi$ with respect to soft rules only is defined as
\[
W'_{\Pi}(I)=
\begin{cases}
  exp\Bigg(\sum\limits_{w:R\;\in\; (\Pi^{\rm soft})_I} w\Bigg) & 
      \text{if $I \in \sm '[\Pi]$}; \\
  0 & \text{otherwise.}
\end{cases}
\]

The normalized weight (a.k.a. probability) of $I$ under $\Pi$ with respect to soft rules only is defined as
\[ 
\small 
  P'_\Pi(I)  = \frac{W'_\Pi(I)}{\sum\limits_{J\in {\rm SM'}[\Pi]}{W'_\Pi(J)}}. 
\]

The proof of \textbf{Theorem \ref{thm:plog2lpmln}} will use the following lemma:

\begin{lemma}\label{lem:soft:prop2} (proposition 2 in \cite{lee16weighted})
If $\sm '[\Pi]$ is not empty, 
	for every interpretation $I$ of $\Pi$, $P'_{\Pi}(I)$ coincides with $P_{\Pi}(I)$.
\end{lemma}

It follows from \textbf{Lemma \ref{lem:soft:prop2}} that if $\sm '[\Pi]$ is not empty, then 
\begin{itemize}
\item
$I$ is a probabilistic stable model of $\Pi$ iff $I \in \sm '[\Pi]$,
\item
every probabilistic stable model of $\Pi$ should satisfy all hard rules in $\Pi$.
\end{itemize}

\medskip

\noindent{\bf Theorem~\ref{thm:plog2lpmln} \optional{thm:plog2lpmln}}\ \ 
{\sl
Let $\Pi$ be a consistent P-log program, $\sigma$ be the signature of $\tau(\Pi)$. There is a 1-1 correspondence $\phi$ between the set of the possible worlds of $\Pi$ with non-zero probabilities and the set of probabilistic stable models of ${\rm plog2lpmln}(\Pi)$ such that 
\bi
\ii[{\bf (a)}] For every possible world $W$ of $\Pi$ that has a non-zero probability, $\phi(W)$ is a probabilistic stable model of ${\rm plog2lpmln(\Pi)}$, and $\mu_\Pi(W) = P_{\rm plog2lpmln(\Pi)}(\phi(W))$.

\ii[{\bf (b)}]  For every probabilistic stable model $I$ of ${\rm plog2lpmln(\Pi)}$, $I|_{\sigma}$ is a possible world of $\Pi$, $I=\phi(I|_{\sigma})$, and $\mu_\Pi(I|_{\sigma})>0$.
\ei
}
\medskip

Note that we make {\bf (b)} a little bit more stronger than the statement in the main body (by adding ``$I=\phi(I|_{\sigma})$'', which is already covered by ``1-1 correspondence''). In this case, to prove {\bf Theorem \ref{thm:plog2lpmln}}, it is sufficient to prove {\bf (a)} and {\bf (b)}.

\medskip
\begin{proof}
For any possible world $W$ of a P-log program $\Pi$, we define the mapping $\phi$ as follows.
\begin{enumerate}
\item
    $\phi(W)\models \i{Poss}_r(c(\vec{u}) \mvis  v)$ iff $c(\vec{u})\mvis v$ is possible in $W$ due to $r$.
    \item
    For each pr-atom
    $pr_r(c(\vec{u}) \mvis  v \ | \ C) = p$
    in $\Pi$, \\
    $\phi(W)\models \i{PossWithAssPr}_{r, C}(c(\vec{u}) \mvis  v)$ iff 
    this pr-atom is applied in $W$.
    \item  
    For each pr-atom $pr_r(c(\vec{u}) \mvis  v \ | \ C) = p$ in~$\Pi$, 
          
    $\phi(W)\models \i{AssPr}_{r, C}(c(\vec{u}) \mvis  v)$ iff this pr-atom is applied in~$W$, and $W\models c(\vec{u})\mvis v$.
    \item
    $\phi(W)\models \i{PossWithAssPr}(c(\vec{u}) \mvis  v)$ iff $v \in  AV_W(c(\vec{u}))$.
    \item
    $\phi(W)\models \i{PossWithDefPr}(c(\vec{u}) \mvis  v)$ iff $c(\vec{u})\mvis v$ is possible in $W$ and $v \not \in  AV_W(c(\vec{u}))$.
    \item
    $\phi(W)\models \i{NumDefPr}(c(\vec{u}), m)$ iff there exist exactly $m$ different values $v$ such that $c(\vec{u})\mvis v$ is possible in $W$; $v \not \in AV_W(c(\vec{u}))$; and, for one of such $v$, $W\models c(\vec{u})\mvis v$.
    \item
    $\phi(W)\models \i{RemPr}(c(\vec{u}), k)$ iff 
there exists a value $v$ such that $W\models c(\vec{u})\mvis v$; $c(\vec{u}) \mvis  v$ is possible in $W$; $v \not \in AV_W(c(\vec{u}))$; and 

    $k = 1 - \sum\limits_{ v\in AV_W(c(\vec{u})) } \i{PossWithAssPr}(W, c(\vec{u}) \mvis  v)$.
    \item  
    $\phi(W)\models \i{TotalDefPr}(c(\vec{u}), k)$ iff $\phi(W)\models \i{RemPr}(c(\vec{u}), k)$ and $k>0$.
\end{enumerate}
Let's denote ${\rm plog2lpmln}(\Pi)$ as $\Pi'$. In the following two parts, we will prove each of the two bullets of \textbf{Theorem \ref{thm:plog2lpmln}}.

\begin{enumerate}
	\item[{\bf (a)}]
	For every possible world $W$ of $\Pi$ with a non-zero probability, to prove $\phi(W)$ is a probabilistic stable model of $\Pi'$, it is sufficient to prove $\phi(W)$ is a stable model of $\overline{\Pi'~^{\rm hard}}$. Indeed, if we prove $\phi(W)$ is a stable model of $\overline{\Pi'~^{\rm hard}}$, then $\phi(W)$ is a stable model of $\overline{\Pi'~^{\rm hard}_{\phi(W)}}$ and $P_{\Pi'~^{\rm hard}}(\phi(W))$ is always greater than 0. Consequently, $\phi(W)$ must be a probabilistic stable model of $\Pi'~^{\rm hard}$.
	 Since $\Pi' = \Pi'~^{\rm hard} \cup \Pi'~^{\rm soft}$, and $\Pi'~^{\rm soft}$ is a set of soft rules of the form ``$w:\ \leftarrow \no\ A$'', where $A$ is an atom and $w$ is a real number, by {\bf Lemma \ref{lem:constr}} (it follows from {\bf Lemma \ref{lem:constr}} that, if all $w$ in $\Pi^{\rm constr}$ are real numbers, $I$ is a probabilistic stable model of $\Pi \cup \Pi^{\rm constr}$ iff $I$ is a probabilistic stable model of $\Pi$), $\phi(W)$ is a probabilistic stable model of $\Pi'$. 

	Let $\sigma$ denote the signature of $\tau(\Pi)$, $\Pi_{AUX} = \overline{\Pi'~^{\rm hard}} \setminus \tau(\Pi)$. 
	It can be seen that no atom in $\sigma$ has a strictly positive occurrence in $\Pi_{AUX}$, and no atom in $\sigma_3$ has a strictly positive occurrence in $\tau(\Pi)$. Furthermore, the construction of $\Pi'$ guarantees that all loops of size greater than one involves atoms in $\sigma$ only. So each strongly connected component of the dependency graph of $\overline{\Pi'~^{\rm hard}}$ relative to $\sigma \cup \sigma_3$ is a subset of $\sigma$ or a subset of $\sigma_3$.
	By the splitting theorem, it is equivalent to show that $\phi(W)$ is a stable model of $\tau(\Pi)$ relative to $\sigma$ and $\phi(W)$ is a stable model of $\Pi_{AUX}$ relative to $\sigma_3$.
	\begin{itemize}
		\item
		\textbf{$\phi(W)$ is a stable model of $\tau(\Pi)$ relative to $\sigma$} :
		Since $W$ is a possible world of $\Pi$, $W$ is a stable model of $\tau(\Pi)$ relative to $\sigma$. Since $\phi(W)$ is an extension of $W$ and no atom in $\phi(W)\setminus W$ belongs to $\sigma$, $\phi(W)$ is a stable model of $\tau(\Pi)$ relative to $\sigma$.
		\item
		\textbf{$\phi(W)$ is a stable model of $\Pi_{AUX}$ relative to $\sigma_3$} :
		Since there is no loop of size greater than one in $\Pi_{AUX}$, we could apply completion on it. Let $\i{Comp}[\Pi_{AUX}; \sigma_3]$ denote the program obtained by applying completion on $\Pi_{AUX}$ with respect to $\sigma_3$, which is as follows:
		\begin{itemize}
		\item
		For each random selection rule (\ref{random}) for $c(\vec{u})$, for each $v \in \i{Range}(c)$ and $x\in\{2, \dots, |\i{Range}(c)| \}$, $\i{Comp}[\Pi_{AUX}; \sigma_3]$ contains:
		{\small
		\beq
		\i{Poss}_r(c(\vec{u}) = v) \leftrightarrow \i{Body} \land p(v) \land \neg \i{Intervene}(c(\vec{u}))
		\eeq {equ:comp_poss}
		}
		{\small
		\beq
		\i{PossWithDefPr}(c(\vec{u}) = v) \leftrightarrow \neg \i{PossWithAssPr}(c(\vec{u}) = v)\ \land \bigvee\limits_{\substack{\text{r' :}\\ \left[ r' \right] \ random(c(\vec{u}): \{ X: p(X) \}) \leftarrow Body \in \Pi}} \i{Poss}_{r'}(c(\vec{u}) = v)
		\eeq {equ:comp1}
		\beq
		\i{NumDefPr}(c(\vec{u}), x) \leftrightarrow 
		x = \#count\{y: \i{PossWithDefPr}(c(\vec{u}) \mvis  y) \} \land
		\bigvee\limits_{\substack{z \in Range(c)}
		} \bigg(c(\vec{u}) = z\land 
		\i{PossWithDefPr}(c(\vec{u}) = z)\bigg)
		\eeq {equ:comp2}
		}
		\item
		For each random selection rule (\ref{random}) for $c(\vec{u})$ along with all pr-atoms associated with it in \textbf{P}:
\beq 
\ba l
    pr_r(c(\vec{u}) \mvis  v_1  \mid  C_1) = p_1   \\
    \dots  \\ 
    pr_r(c(\vec{u}) \mvis  v_n  \mid  C_n) = p_n 
\ea
\eeq {pr1}
		where $n \geq 1$, for $i \in \{ 1, \dots, n \}$, $\i{Comp}[\Pi_{AUX}; \sigma_3]$ also contains:
		{\small
		\beq
		\i{\i{PossWithAssPr}}_{r, C_i}(c(\vec{u}) = v_i) \leftrightarrow \i{Poss}_r(c(\vec{u}) = v_i) \land C_i 
		\eeq {equ:comp3}
		\beq
		\i{AssPr}_{r, C_i}(c(\vec{u}) = v_i) \leftrightarrow c(\vec{u}) = v_i \land \i{\i{PossWithAssPr}}_{r, C_i}(c(\vec{u}) = v_i) 
		\eeq {equ:comp5}
		\beq
		\neg \i{AssPr}_{r, C_i} (c(\vec{u}) = v_i) \hspace{20pt} \text{(if $p_i = 0$)}
		\eeq {comp_constr_asspr}
		\beq
		\i{PossWithAssPr}(c(\vec{u}) = v_i) \leftrightarrow \bigvee\limits_{\substack{\text{r', j :}\\pr_{r'}(c(\vec{u}) = v_i \mid C_j) = p_j \in \Pi}}  \i{\i{PossWithAssPr}}_{r', C_j}(c(\vec{u}) = v_i)
		\eeq {equ:comp4}
		}
		\item
		For each $c(\vec{u})$ in {\bf S} and $x\in {\bf p}_{rem}(c(\vec{u}), \Pi)$, $\i{Comp}[\Pi_{AUX}; \sigma_3]$ also contains:
		{\small
		\beq
		\ba {l l}
		\i{RemPr}(c(\vec{u}), x) \leftrightarrow &\bigvee\limits_{v \in Range(c)} \bigg( c(\vec{u}) = v\land \i{PossWithDefPr}(c(\vec{u}) = v) \bigg) \land \\
		
		& \bigvee\limits_{\substack{\text{r' :}\\ \left[ r' \right] \ random(c(\vec{u}): \{ X: p(X) \}) \leftarrow Body \in \Pi}}
		\bigg( \i{Body} \land x=1\!-\!y \ \land \\
		& y = \#{\rm sum}\{p_1:\i{PossWithAssPr}_{r', C_1}(c(\vec{u}) \mvis  v_1);\dots ;
p_n:\i{PossWithAssPr}_{r', C_n}(c(\vec{u}) \mvis  v_n)\}  \bigg)
		\ea
		\eeq {comp_rempr}
		\beq
		\i{TotalDefPr}(c(\vec{u}), x) \leftrightarrow \i{RemPr}(c(\vec{u}), x) \land x>0
		\eeq {comp_totaldefpr}
		}
		{\small
		\beq
		\neg (\i{RemPr}(c(\vec{u}), x) \land x\leq 0 )
		\eeq {comp_constr_totaldefpr}
		}
		
		\end{itemize}
		First, let's expand some notations in the definition of $\phi(W)$:
		\begin{itemize}
		\item
		$c(\vec{u})=v$ is possible in $W$\\
		By definition, it is equivalent to ``there exists a random selection rule (\ref{random}) such that $W\vDash \i{Body} \land p(v) \land \neg \i{Intervene}(c(\vec{u}))$''.
		\item
		a pr-atom $pr_r (c(\vec{u}) = v_i\mid C_i) = p_i$ is applied in $W$\\
		By definition, it is equivalent to ``$c(\vec{u}) = v_i$ is possible in $W$ due to $r$, and $W\vDash C_i$''.
		\item
		$v \in  AV_W(c(\vec{u}))$\\
		By the definition of $AV_W(c(\vec{u}))$, it is equivalent to ``there exists a pr-atom
		$pr_r (c(\vec{u}) = v\mid C_i) = p_i$ that is applied in $W$ for some $r$ and $i$
		''. 
		\end{itemize}

$\newline$
		Then we will prove that each formula in $\i{Comp}[\Pi_{AUX}; \sigma_3]$ is satisfied by $\phi(W)$ based on the definition of $\phi(W)$:
		\begin{itemize}
		\item
		{\bf Let's take formula (\ref{equ:comp_poss}) into account.} Consider the random selection rule
		$\left[ r \right] \ random(c(\vec{u}): \{ X: p(X) \}) \leftarrow \i{Body}$, where formula (\ref{equ:comp_poss}) is obtained. By definition, 
		\begin{itemize}
		\item
		$\phi(W)\vDash \i{Poss}_r(c(\vec{u}) = v)$ 
		\end{itemize}
		iff 
		\begin{itemize}
		\item
		$c(\vec{u}) = v$ is possible in $W$ due to $r$
		\end{itemize}
		iff 
		\begin{itemize}
		\item
		$W\vDash \i{Body}\land p(v) \land \no\ \i{Intervene}(c(\vec{u}))$
		\end{itemize}
		iff (since $\phi(W)$ is an extension of $W$)
		\begin{itemize}
		\item
		$\phi(W)\vDash \i{Body}\land p(v) \land \no\ \i{Intervene}(c(\vec{u}))$
		\end{itemize}
		Thus formula (\ref{equ:comp_poss}) is satisfied by $\phi(W)$.
		
		\item
		{\bf Let's take formula (\ref{equ:comp1}) into account.} By definition, 
		\begin{itemize}
		\item
		$\phi(W)\vDash \i{PossWithDefPr}(c(\vec{u}) = v)$ 
		\end{itemize}
		iff 
		\begin{itemize}
		\item
		$c(\vec{u}) = v$ is possible in $W$
		\item
		$v \not \in AV_W(c(\vec{u}))$
		\end{itemize}
		iff (by definition)
		\begin{itemize}
		\item
		there exists a random selection rule $r$ such that $c(\vec{u}) = v$ is possible in $W$ due to $r$
		\item
		$\phi(W) \not \vDash \i{PossWithAssPr}(c(\vec{u}) = v)$
		\end{itemize}
		iff (by definition)
		\begin{itemize}
		\item
		there exists a random selection rule $r$ such that $\phi(W) \vDash \i{Poss}_{r}(c(\vec{u}) = v)$
		\item
		$\phi(W)\vDash \neg \i{PossWithAssPr}(c(\vec{u}) = v)$
		\end{itemize}
		Thus formula (\ref{equ:comp1}) is satisfied by $\phi(W)$.
		
		\item
		{\bf Let's take formula (\ref{equ:comp2}) into account.} By definition, 
		\begin{itemize}
		\item
		$\phi(W)\vDash \i{NumDefPr}(c(\vec{u}), x)$ 
		\end{itemize}
		iff 
		\begin{itemize}
		\item
		there exist exactly $x$ different $v$ such that 
			\begin{itemize}
			\item
			$c(\vec{u})=v$ is possible in $W$
			\item
			$v \not \in AV_W(c(\vec{u}))$
			\item
			for one of such $v$, $W\models c(\vec{u})=v$
			\end{itemize}
		\end{itemize}
		iff (by definition and since $\phi(W)$ is an extension of $W$)
		\begin{itemize}
		\item
		there exists exactly $x$ different $v$ such that 
			\begin{itemize}
			\item
			$\phi(W)\vDash \i{PossWithDefPr}(c(\vec{u}) = v)$
			\item
			for one of such $v$, $\phi(W)\vDash c(\vec{u})=v \land \i{PossWithDefPr}(c(\vec{u}) = v)$
			\end{itemize}
		\end{itemize}
		Thus formula (\ref{equ:comp2}) is satisfied by $\phi(W)$.
		
		\item
		{\bf Let's take formula (\ref{equ:comp3}) into account.} Consider the pr-atom
		$pr_r(c(\vec{u}) = v_i \mid C_i) = p_i$
		where formula (\ref{equ:comp3}) is obtained. By definition, 
		\begin{itemize}
		\item
		$\phi(W)\vDash \i{PossWithAssPr}_{r, C_i}(c(\vec{u}) = v_i)$ 
		\end{itemize}
		iff 
		\begin{itemize}
		\item
		this pr-atom is applied in $W$
		\end{itemize}
		iff 
		\begin{itemize}
		\item
		$c(\vec{u}) = v_i$ is possible in $W$ due to $r$, and $W\vDash C_i$
		\end{itemize}
		iff (by definition and since $\phi(W)$ is an extension of $W$)
		\begin{itemize}
		\item
		$\phi(W)\vDash \i{Poss}_r(c(\vec{u}) = v_i) \land C_i$
		\end{itemize}
		Thus formula (\ref{equ:comp3}) is satisfied by $\phi(W)$.
		
		{\bf Remark:} By {\bf Condition 1}, $r$ is the only random selection rule for $c(\vec{u})$ whose ``\i{Body}'' is satisfied by $W$. And by {\bf Condition 2}, there won't be another pr-atom $pr_r(c(\vec{u}) = v \ | \ C') = p' \in \Pi$ such that $W\vDash C'$. Thus for any $c(\vec{u}) = v$, $\phi(W)$ could at most satisfy one $\i{PossWithAssPr}_{r, C_i}(c(\vec{u}) = v_i)$ for any $r$ and $C_i$.
		
		\item
		{\bf Let's take formula (\ref{equ:comp5}) into account.} Consider the pr-atom
		$pr_r(c(\vec{u}) = v_i \mid C_i) = p_i$
		in $\Pi$, where formula (\ref{equ:comp5}) is obtained, by definition, 
		\begin{itemize}
		\item
		$\phi(W)\vDash \i{AssPr}_{r, C_i}(c(\vec{u}) = v_i)$ 
		\end{itemize}
		iff 
		\begin{itemize}
		\item
		this pr-atom is applied in $W$ 
		\item
		$W \vDash c(\vec{u}) = v_i$
		\end{itemize}
		iff (by definition and since $\phi(W)$ is an extension of $W$)
		\begin{itemize}
		\item
		$\phi(W)\vDash \i{PossWithAssPr}_{r, C_i}(c(\vec{u}) = v_i) \land c(\vec{u})=v_i$
		\end{itemize}
		Thus formula (\ref{equ:comp5}) is satisfied by $\phi(W)$.
		
		\item
		{\bf Let's take formula (\ref{comp_constr_asspr}) into account.} For any pr-atom
		$pr_r(c(\vec{u}) = v_i \mid C_i) = p_i$
		in $\Pi$ such that $p_i = 0$, assume for the sake of contradiction that $\phi(W) \vDash \i{AssPr}_{r, C_i}(c(\vec{u}) = v_i)$. Then by definition, this pr-atom is applied and $W\vDash c(\vec{u}) = v_i$. In other words, $c(\vec{u}) = v_i \in W$, $c(\vec{u}) = v_i$ is possible in $W$, and $P(W,c(\vec{u}) = v_i)=0$. Thus $\hat{\mu}_\Pi(W) = 0$, which contradicts that $\mu_\Pi(W) > 0$.
		
		Thus formula (\ref{comp_constr_asspr}) is satisfied by $\phi(W)$.
		
		\item
		{\bf Let's take formula (\ref{equ:comp4}) into account.} By definition, 
		\begin{itemize}
		\item
		$\phi(W)\vDash \i{PossWithAssPr}(c(\vec{u}) = v_i)$ 
		\end{itemize}
		iff 
		\begin{itemize}
		\item
		$v_i \in AV_W(c(\vec{u}))$
		\end{itemize}
		iff 
		\begin{itemize}
		\item
		there exist a pr-atom $pr_r(c(\vec{u}) = v_i \mid C_j) = p_j$ that is applied in $W$ for some $r$ and $j$ (where $i$ and $j$ may be different)
		\end{itemize}
		iff (by definition)
		\begin{itemize}
		\item
		there exist $r$ and $j$ such that $\phi(W)\vDash \i{PossWithAssPr}_{r, C_j}(c(\vec{u}) = v_i)$
		\end{itemize}
		Thus formula (\ref{equ:comp4}) is satisfied by $\phi(W)$.
		
		\item
		{\bf Let's take formula (\ref{comp_rempr}) into account.} By definition,
		\begin{itemize}
		\item
		$\phi(W)\vDash \i{RemPr}(c(\vec{u}), x)$ 
		\end{itemize}
		iff 
		\begin{itemize}
		\item
		there exists a $v$ such that
			\begin{itemize}
			\item
			$W\vDash c(\vec{u}) \mvis  v$
			\item
			$c(\vec{u}) \mvis  v$ is possible in $W$
			\item
			$v \not \in AV_W(c(\vec{u}))$, and
			\end{itemize}
		\item
		$x = 1 - \sum\limits_{ v'\in AV_W(c(\vec{u})) } \i{PossWithAssPr}(W, c(\vec{u}) \mvis  v')$
		\end{itemize}
		iff (by definition and since $\phi(W)$ is an extension of $W$)
		\begin{itemize}
		\item
		there exists a $v$ such that $\phi(W)\vDash c(\vec{u}) = v\land \i{PossWithDefPr}(c(\vec{u}) = v)$, 
		\item
		$x = 1 - y$, and
		\item
		$y = \sum\limits_{ \phi(W)\vDash PossWithAssPr(c(\vec{u}) = v') } \i{PossWithAssPr}(W, c(\vec{u}) \mvis  v')$
		\end{itemize}
		iff (by formula (34) and the definition of $\i{PossWithAssPr}(W, c(\vec{u}) \mvis  v)$)
		\begin{itemize}
		\item
		there exists a $v$ such that $\phi(W)\vDash c(\vec{u}) = v\land \i{PossWithDefPr}(c(\vec{u}) = v)$, 
		\item
		$x = 1 - y$, and
		\item
		there exists a random selection rule $r$ \eqref{random} along with all pr-atoms \eqref{pr1} associated with it such that
			\begin{itemize}
			\item
			$\phi(W) \vDash \i{Body}$
			\item
			$y = \sum\limits_{ j: \phi(W) \vDash PossWithAssPr_{r, C_j}(c(\vec{u}) \mvis  v_j) } p_j$
			\end{itemize}
		\end{itemize}
		Thus formula (\ref{comp_rempr}) is satisfied by $\phi(W)$.
		
		{\bf Remark:} By {\bf Condition 1}, there exits at most one random selection rule whose ``\i{Body}'' is satisfied by $W$. Thus there is no other random selection rule $r'$ such that $\phi(W) \vDash \i{PossWithAssPr}_{r', C_j}(c(\vec{u}) \mvis  v_j)$ for any $j$.
		Morevoer, for any $c(\vec{u}) \in \Pi$, there exits at most one $\i{RemPr}(c(\vec{u}), x)$ that can be satisfied by $\phi(W)$.
		
		\item
		{\bf Let's take formula (\ref{comp_totaldefpr}) into account.} By definition, 
		\begin{itemize}
		\item
		$\phi(W)\vDash \i{TotalDefPr}(c(\vec{u}), x)$ 
		\end{itemize}
		iff 
		\begin{itemize}
		\item
		$\phi(W)\vDash \i{RemPr}(c(\vec{u}), x) \land x>0$
		\end{itemize}
		Thus formula (\ref{comp_totaldefpr}) is satisfied by $\phi(W)$.
		
		\item
		{\bf Let's take formula (\ref{comp_constr_totaldefpr}) into account.} Consider the random selection rule (\ref{random}) for $c(\vec{u})$ along with all pr-atoms \eqref{pr1} associated with it ($n \geq 1$), where formula (\ref{comp_constr_totaldefpr}) is obtained. Assume for the sake of contradiction that $\phi(W) \vDash (\i{RemPr}(c(\vec{u}), x) \land x\leq 0 )$ for some $x$, which (by definition) follows that 
		\begin{itemize}
		\item
		there exists a $v$ such that
			\begin{itemize}
			\item
			$W\vDash c(\vec{u}) \mvis  v$
			\item
			$c(\vec{u}) \mvis  v$ is possible in $W$
			\item
			$v \not \in AV_W(c(\vec{u}))$
			\end{itemize}
		\item
		$x = 1 - \sum\limits_{ v\in AV_W(c(\vec{u})) } \i{PossWithAssPr}(W, c(\vec{u}) \mvis  v)$ and $x \leq 0$
		\end{itemize}
		In other words, $c(\vec{u}) \mvis v \in W$, $c(\vec{u}) \mvis v$ is possible in $W$, and $P(W, c(\vec{u}) \mvis v)= \i{PossWithDefPr}(W,c(\vec{u}) \mvis v) = 0$. Thus $\hat{\mu}_\Pi(W) = 0$, which contradicts that $\mu_\Pi(W) > 0$.
		
		Thus formula (\ref{comp_constr_totaldefpr}) is satisfied by $\phi(W)$.
	\end{itemize}
	\medskip
		Now we see the definition of $\phi(W)$ guarantees that $\phi(W)$ is a model of  $\i{Comp}[\Pi_{AUX}; \sigma_3]$. Thus $\phi(W)$ is a stable model of  $\Pi_{AUX}$ relative to $\sigma_3$.
	\end{itemize}
	
	Until now we proved $\phi(W)$ is a stable model of $\Pi'$.
	Then, we are going to prove $\mu_{\Pi}(W) = P_{\Pi'}(\phi(W))$. 
	
	Recall that $\Pi'$ denotes the translated $\lpmln$ program ${\rm plog2lpmln}(\Pi)$, $W'_{\Pi}(I)$ denotes the unnormalized weight of $I$ under $\Pi$ with respect to soft rules only.
	
	Firstly we will prove $\hat{\mu}_{\Pi}(W) = W'_{\Pi'}(\phi(W))$.
	From the definition of $\hat{\mu}_{\Pi}(W)$ (unnormalized probability) and $P(W,c(\vec{u})=v)$ in the semantics of P-log, we have
	\begin{align*}
	\hat{\mu}_{\Pi}(W) &= \prod_{\substack{\text{$c(\vec{u})=v:$}\\\text{$c(\vec{u})=v$ is possible in $W$}\\\text{and $W\vDash c(\vec{u})=v$}}}P(W,c(\vec{u})=v)
	\\
	& = \prod_{\substack{\text{$c(\vec{u})=v:$}\\\text{$c(\vec{u})=v$ is possible in $W$}\\\text{$W\vDash c(\vec{u})=v$}\\\text{and $v \in AV_W(c(\vec{u}))$}}}P(W,c(\vec{u})=v) 
	\times 
	\prod_{\substack{\text{$c(\vec{u})=v:$}\\\text{$c(\vec{u})=v$ is possible in $W$}\\\text{$W\vDash c(\vec{u})=v$}\\\text{and $v \not \in AV_W(c(\vec{u}))$}}}P(W,c(\vec{u})=v)
	\\
	& = \prod_{\substack{\text{$c(\vec{u})=v:$}\\\text{$c(\vec{u})=v$ is possible in $W$}\\\text{$W\vDash c(\vec{u})=v$}\\\text{and $v \in AV_W(c(\vec{u}))$}}}\i{PossWithAssPr}(W,c(\vec{u})=v) 
	\times \\
	&\hspace{12pt} \prod_{\substack{\text{$c(\vec{u})=v:$}\\\text{$c(\vec{u})=v$ is possible in $W$}\\\text{and $W\vDash c(\vec{u})=v$}\\\text{and $v \not \in AV_W(c(\vec{u}))$}}}\i{PossWithDefPr}(W,c(\vec{u})=v)
	\end{align*}
	Since $W$ is a possible world of $\Pi$ with a non-zero probability, $\hat{\mu}_{\Pi}(W) > 0$.
	Since the statement ``$v \in AV_W(c(\vec{u}))$'' is equivalent to saying ``$c(\vec{u})=v$ is possible in $W$, and there exists $pr_{r_{W, c(\vec{u})}} (c(\vec{u}) = v\mid C) = p \in \Pi$ for some $C$ and $p$, and $W\vDash C$'', we have
	\begin{align*}
	\hat{\mu}_{\Pi}(W) &= \prod_{\substack{\text{$c(\vec{u})=v:$}\\\text{$c(\vec{u})=v$ is possible in $W$}\\\text{$W\vDash c(\vec{u})=v$}\\\text{$pr_{r_{W, c(\vec{u})}} (c(\vec{u}) = v\mid C) = p \in \Pi$}\\\text{and $W\vDash C$}}}p
	\times \\
	&\hspace{12pt} \prod_{\substack{\text{$c(\vec{u})=v:$}\\\text{$c(\vec{u})=v$ is possible in $W$}\\\text{$W\vDash c(\vec{u})=v$}\\\text{and $v \not \in AV_W(c(\vec{u}))$}}}\frac{1 - \sum_{v'\in AV_W(c(\vec{u}))}\i{PossWithAssPr}(W, c(\vec{u}) = v')}{|\{v'' \mid c(\vec{u})\mvis v'' \text{ is possible in } W \text{ and } v'' \not\in AV_W(c(\vec{u})) \}|}
	\\
	& = \prod_{\substack{\text{$c(\vec{u})=v:$}\\\text{$c(\vec{u})=v$ is possible in $W$}\\\text{$W\vDash c(\vec{u})=v$}\\\text{$pr_{r_{W, c(\vec{u})}} (c(\vec{u}) = v\mid C) = p \in \Pi$}\\\text{and $W\vDash C$}}}p
	\times \\
	&\hspace{12pt} \prod_{\substack{\text{$c(\vec{u})=v:$}\\\text{$c(\vec{u})=v$ is possible in $W$}\\\text{$W\vDash c(\vec{u})=v$}\\\text{and $v \not \in AV_W(c(\vec{u}))$}}}\frac{1}{|\{v' \mid c(\vec{u})\mvis v' \text{ is possible in } W \text{ and } v' \not\in AV_W(c(\vec{u})) \}|}
	\times \\
	&\hspace{12pt} \prod_{\substack{\text{$c(\vec{u})=v:$}\\\text{$c(\vec{u})=v$ is possible in $W$}\\\text{$W\vDash c(\vec{u})=v$}\\\text{and $v \not \in AV_W(c(\vec{u}))$}}}(1 - \sum\limits_{\substack{
	\text{$c(\vec{u})=v':$} \\
	\text{$c(\vec{u})=v'$ is possible in $W$} \\
	\text{$pr_{r_{W, c(\vec{u})}} (c(\vec{u}) = v'\mid C) = p \in \Pi$} \\
	\text{and $W\vDash C$} \\
	}}p)
	\end{align*}
	Note that by \textbf{Condition 1}, the subscript $r_{W, c(\vec{u})}$ of the applied pr-atom is the only random selection rule for $c(\vec{u})$ whose body could be satisfied by $W$.
	
	We then calculate $W'_{\Pi'}(\phi(W))$, the unnormalized weight of $\phi(W)$ with respect to all soft rules in $\Pi'$. From the construction of $\Pi'$, it's easy to see that there are only 3 kinds of soft rules: Rule (\ref{w-asspr}), Rule (\ref{w-numdefpr}), and Rule (\ref{w-totaldefpr}), which are satisfied iff $\phi(W)\vDash \i{AssPr}_{r, C}(c(\vec{u})=v)$, $\phi(W)\vDash \i{NumDefPr}(c(\vec{u}),m)$, and $\phi(W)\vDash \i{TotalDefPr}(c(\vec{u}),x)$, respectively. 
	Let's denote the unnormalized weight of $\phi(W)$ with respect to each of these three rules as $W'_{\Pi'}(\phi(W))|_{\ref{w-asspr}}$, $W'_{\Pi'}(\phi(W))|_{\ref{w-numdefpr}}$, $W'_{\Pi'}(\phi(W))|_{\ref{w-totaldefpr}}$. It's clear that $W'_{\Pi'}(\phi(W)) = W'_{\Pi'}(\phi(W))|_{\ref{w-asspr}} \times W'_{\Pi'}(\phi(W))|_{\ref{w-numdefpr}} \times W'_{\Pi'}(\phi(W))|_{\ref{w-totaldefpr}}$.
	
	Consider a $c(\vec{u}) = v$ that is possible in $W$ and $W \vDash c(\vec{u}) = v$. Since $\hat{\mu}_{\Pi}(W) > 0$, if $v\in AV_W(c(\vec{u}))$, $pr_{r_{W, c(\vec{u})}} (c(\vec{u}) = v\mid C) = p \in \Pi$ and $W\vDash C$, then $P(W,c(\vec{u}) = v) = p$ and $p>0$;
	if f $v \not\in AV_W(c(\vec{u}))$, then $1 - \sum\limits_{ v'\in AV_W(c(\vec{u})) } \i{PossWithAssPr}(W, c(\vec{u}) \mvis  v')$ must be greater than $0$.
	By the definition of $\phi(W)$, 
	\begin{align*}
	W'_{\Pi'}(\phi(W))|_{\ref{w-asspr}} & =  exp\Bigg(
	\sum\limits_{\substack{
	\text{$c(\vec{u})=v:$}\\
	\text{$pr_r (c(\vec{u}) = v\mid C) = p \in \Pi$}\\
	\text{$\phi(W)\vDash \i{AssPr}_{r, C}(c(\vec{u})=v)$}}} ln(p) 
	\Bigg) \\
	& \text{(Note that by \textbf{Condition 1}, $r$ must be the same as $r_{W, c(\vec{u})}$)}\\
	& = \prod_{\substack{
	\text{$c(\vec{u})=v:$}\\
	\text{$pr_{r_{W, c(\vec{u})}} (c(\vec{u}) = v\mid C) = p \in \Pi$}\\
	\text{$c(\vec{u})=v$ is possible in $W$}\\
	\text{$W\vDash C$}\\
	\text{and $W\vDash c(\vec{u})=v$}}}p
	\end{align*}

	\begin{align*}
	W'_{\Pi'}(\phi(W))|_{\ref{w-numdefpr}} & =  exp\Bigg(
	\sum\limits_{\substack{
	\text{$c(\vec{u}),m:$}\\
	\text{$m\geq 2$}\\
	\text{$\phi(W)\vDash \i{NumDefPr}(c(\vec{u}),m)$}}} ln(\frac{1}{m})
	\Bigg) \\
	& =  exp\Bigg(
	\sum\limits_{\substack{
	\text{$c(\vec{u}),m:$}\\
	\text{$\phi(W)\vDash \i{NumDefPr}(c(\vec{u}),m)$}}} ln(\frac{1}{m})
	\Bigg) \\
	& = \prod_{\substack{\text{$c(\vec{u})=v:$}\\\text{$c(\vec{u})=v$ is possible in $W$}\\\text{$W\vDash c(\vec{u})=v$}\\\text{and $v \not \in AV_W(c(\vec{u}))$}}}\frac{1}{|\{v' \mid c(\vec{u})\mvis v' \text{ is possible in } W \text{ and } v' \not\in AV_W(c(\vec{u})) \}|}
	\end{align*}
	\begin{align*}
	W'_{\Pi'}(\phi(W))|_{\ref{w-totaldefpr}} & =  exp\Bigg(
	\sum\limits_{\substack{
	\text{$c(\vec{u}),x:$}\\
	\phi(W) \vDash TotalDefPr(c(\vec{u}),x)
	}}ln(x)
	\Bigg) \\
	& = exp\Bigg(
	\sum\limits_{\substack{
	\text{$c(\vec{u})=v:$}\\
	\text{$c(\vec{u})=v$ is possible in $W$}\\
	\text{$v\not\in AV_W(c(\vec{u}))$}\\
	\text{and $W\vDash c(\vec{u})=v$}}} 
	ln(1 - \sum\limits_{ v'\in AV_W(c(\vec{u})) } \i{PossWithAssPr}(W, c(\vec{u}) \mvis  v')
	\Bigg) \\
	& = exp\Bigg(
	\sum\limits_{\substack{
	\text{$c(\vec{u})=v:$}\\
	\text{$c(\vec{u})=v$ is possible in $W$}\\
	\text{$v\not\in AV_W(c(\vec{u}))$}\\
	\text{and $W\vDash c(\vec{u})=v$}}} 
	ln(1-\sum\limits_{\substack{
	\text{$c(\vec{u})=v':$}\\
	\text{$c(\vec{u})=v'$ is possible in $W$}\\
	\text{$pr_{r_{W, c(\vec{u})}} (c(\vec{u}) = v'\mid C) = p \in \Pi$}\\
	\text{and $W\vDash C$}}}{p})
	\Bigg) \\
	& = \prod_{\substack{\text{$c(\vec{u})=v:$}\\\text{$c(\vec{u})=v$ is possible in $W$}\\\text{$W\vDash c(\vec{u})=v$}\\\text{and $v \not \in AV_W(c(\vec{u}))$}}}(1 - \sum\limits_{\substack{
	\text{$c(\vec{u})=v':$} \\
	\text{$c(\vec{u})=v'$ is possible in $W$} \\
	\text{$pr_{r_{W, c(\vec{u})}} (c(\vec{u}) = v'\mid C) = p \in \Pi$} \\
	\text{and $W\vDash C$} \\
	}}p)
	\end{align*}
	
	It's easy to see that $W'_{\Pi'}(\phi(W))= W'_{\Pi'}(\phi(W))|_{\ref{w-asspr}} \times W'_{\Pi'}(\phi(W))|_{\ref{w-numdefpr}} \times W'_{\Pi'}(\phi(W))|_{\ref{w-totaldefpr}} = \hat{\mu}_{\Pi}(W)$. 
	We already proved that for any possible world $W$ of $\Pi$, $\phi(W)$ is a probabilistic stable model of $\Pi'$. Then to prove $\mu_{\Pi}(W) = P_{\Pi'}(\phi(W))$, it is sufficient to prove for any probabilistic stable model $I$ of $\Pi'$, $I|_{\sigma}$ is a possible world of $\Pi$ and $I = \phi(I|_{\sigma})$ (which will be proved in the next part). Indeed, if we proved this, we know $\phi(W)$ and $W$ are 1-1 correspondent, thus $P'_{\Pi'}(\phi(W)) = \mu_{\Pi}(W)$. Since $\phi(W) \in \sm '[\Pi']$, by {\bf Lemma \ref{lem:soft:prop2}}, $P_{\Pi'}(\phi(W))=P'_{\Pi'}(\phi(W))=\mu_{\Pi}(W)$.
	
	\item[{\bf (b)}]
	Since $\Pi$ is consistent, there exists a possible world $W'$ of $\Pi$ with a non-zero probability. It's proved that $\phi(W')$ is a probabilistic stable model of $\Pi'$ and $\phi(W')$ satisfies $\overline{\Pi'~^{\rm hard}}$. So $\sm '[\Pi']$ is not empty. Let $I$ be a probabilistic stable model of $\Pi'$, by {\bf Lemma \ref{lem:soft:prop2}}, $I\vDash \overline{\Pi'~^{\rm hard}}$. Besides, since $\Pi' \setminus \Pi'~^{\rm hard}$ is a set of rules of the form $w: \leftarrow F$, by \textbf{Lemma \ref{lem:constr}}, $I$ is a stable model of $\overline{\Pi'~^{\rm hard}_I}$. Thus $I$ is a stable model of $\overline{\Pi'~^{\rm hard}}$.
	
	Since (1) $I$ is a stable model of $\tau(\Pi) \cup \Pi_{AUX}$, (2) no atom in $\sigma$ has a strictly positive occurrence in $\Pi_{AUX}$, (3) no atom in $\sigma_3$ has a strictly positive occurrence in $\tau(\Pi)$, (4) each strongly connected component of the dependency graph of $\tau(\Pi) \cup \Pi_{AUX}$ relative to $\sigma \cup \sigma_3$ is a subset of $\sigma$ or a subset of $\sigma_3$, by the splitting theorem
	\begin{itemize}
	\item
	$I$ is a stable model of $\tau(\Pi)$ relative to $\sigma$. Thus $I|_{\sigma}$ is a stable model of $\tau(\Pi)$, which means $I|_{\sigma}$ is a possible world of $\Pi$.
	\item
	$I$ is a stable model of $\Pi_{AUX}$ relative to $\sigma_3$. So $I\vDash \i{Comp}[\Pi_{AUX}; \sigma_3]$. 
	\end{itemize}
	Let's denote $I|_{\sigma}$ by $W$, we'll prove $I = \phi(W)$ by checking if $I$ satisfies all conditions in the definition of $\phi(W)$. 
	\begin{itemize}
	\item
	Let's consider condition \textbf{(1)} in the definition of $\phi$. 
	Take any random selection rule
	$\left[ r \right] \ random(c(\vec{u}): \{ X: p(X) \}) \leftarrow \i{Body}$,
	since $I$ satisfies formula (\ref{equ:comp_poss}), 
	\begin{itemize}
	\item
	$I\vDash \i{Poss}_{r}(c(\vec{u}) = v)$ 
	\end{itemize}
	iff 
	\begin{itemize}
	\item
	$I\vDash \i{Body} \land p(v) \land \neg \i{Intervene}(c(\vec{u}))$
	\end{itemize}
	iff (since all atoms in the above conjunction part belong to $\sigma$)
	\begin{itemize}
	\item
	$W\vDash \i{Body} \land p(v) \land \neg \i{Intervene}(c(\vec{u}))$
	\end{itemize}
	iff 
	\begin{itemize}
	\item
	$c(\vec{u}) = v$ is possible in $W$ due to $r$.
	\end{itemize}
	
	\item
	Let's consider condition \textbf{(2)} in the definition of $\phi$. 
	Take any pr-atom
	$pr_r(c(\vec{u}) = v_i \mid C_i) = p_i$
	in $\Pi$, 
	since $I$ satisfies formula (\ref{equ:comp3}), 
	\begin{itemize}
	\item
	$I\vDash \i{PossWithAssPr}_{r, C_i}(c(\vec{u}) = v_i)$ 
	\end{itemize}
	iff 
	\begin{itemize}
	\item
	$I\vDash \i{Poss}_r(c(\vec{u}) = v_i) \land C_i$
	\end{itemize}
	iff (from the proof of condition {\bf (1)}, and since $C_i$ belongs to $\sigma$)
	\begin{itemize}
	\item
	$c(\vec{u}) = v_i$ is possible in $W$ due to $r$ and $W\vDash C_i$
	\end{itemize}
	iff 
	\begin{itemize}
	\item
	this pr-atom is applied in $W$
	\end{itemize}

	Thus condition \textbf{(2)} is satisfied by $I$.
	\item
	Let's consider condition \textbf{(3)} in the definition of $\phi$. 
	Take any pr-atom
	$pr_r(c(\vec{u}) = v_i \mid C_i) = p_i$
	in $\Pi$, 
	since $I$ satisfies formula (\ref{equ:comp5}), 
	\begin{itemize}
	\item
	$I\vDash \i{AssPr}_{r, C_i}(c(\vec{u}) = v_i)$ 
	\end{itemize}
	iff 
	\begin{itemize}
	\item
	$I\vDash \i{PossWithAssPr}_{r, C_i}(c(\vec{u}) = v_i) \land c(\vec{u})=v_i$
	\end{itemize}
	iff (from the proof of condition {\bf (2)}, and since $c(\vec{u}) = v_i$ belongs to $\sigma$)
	\begin{itemize}
	\item
	this pr-atom is applied in $W$
	\item
	$W\vDash c(\vec{u}) = v_i$
	\end{itemize}

	Thus condition \textbf{(3)} is satisfied by $I$.
	\item
	Let's consider condition \textbf{(4)} in the definition of $\phi$. 
	Since $I$ satisfies formula (\ref{equ:comp4}), 
	\begin{itemize}
	\item
	$I\vDash \i{PossWithAssPr}(c(\vec{u}) = v_i)$ 
	\end{itemize}
	iff (from the proof of condition {\bf (2)})
	\begin{itemize}
	\item
	there exist a $r$ and $j$ such that $I\vDash \i{\i{PossWithAssPr}}_{r, C_j}(c(\vec{u}) = v_i)$
	\end{itemize}
	iff 
	\begin{itemize}
	\item
	there exist a pr-atom $pr_r(c(\vec{u}) = v_i \mid C_j) = p_j$ that is applied in $W$ for some $r$ and $j$ (where $i$ and $j$ may be different)
	\end{itemize}
	iff 
	\begin{itemize}
	\item
	$v_i \in AV_W(c(\vec{u}))$
	\end{itemize}

	Thus condition \textbf{(4)} is satisfied by $I$.
	\item
	Let's consider condition \textbf{(5)} in the definition of $\phi$. 
	Since $I$ satisfies formula (\ref{equ:comp1}), 
	\begin{itemize}
	\item
	$I\vDash \i{PossWithDefPr}(c(\vec{u}) = v)$ 
	\end{itemize}
	iff 
	\begin{itemize}
	\item
	$I\vDash \neg \i{PossWithAssPr}(c(\vec{u}) = v)$ 
	\item
	there exists a random selection rule $\left[ r \right] \ random(c(\vec{u}): \{ X: p(X) \}) \leftarrow \i{Body}$, such that 
	$I \vDash \i{Poss}_{r}(c(\vec{u}) = v)$
	\end{itemize}
	iff (by condition {\bf (4)} and  {\bf (1)})
	\begin{itemize}
	\item
	$v \not \in AV_W(c(\vec{u}))$
	\item
	$c(\vec{u}) = v$ is possible in $W$
	\end{itemize}

	Thus condition \textbf{(5)} is satisfied by $I$.
	\item
	Let's consider condition \textbf{(6)} in the definition of $\phi$. 
	Since $I$ satisfies formula (\ref{equ:comp2}), 
	\begin{itemize}
	\item
	$I\vDash \i{NumDefPr}(c(\vec{u}), x)$ 
	\end{itemize}
	iff 
	\begin{itemize}
	\item
	$x = \#count\{y: \i{PossWithDefPr}(c(\vec{u}) \mvis  y) \}$
	\item
	there exists a $c(\vec{u}) = z$ such that $I\vDash c(\vec{u}) = z\land \i{PossWithDefPr}(c(\vec{u}) = z)$
	\end{itemize}
	iff 
	\begin{itemize}
	\item
	there exist exactly $x$ different values $v$ such that 
		\begin{itemize}
		\item
		$I\vDash \i{PossWithDefPr}(c(\vec{u}) = v)$
		\item
		for one of such $v$, $I\vDash c(\vec{u}) = v$
		\end{itemize}
	\end{itemize}
	iff (by condition {\bf (5)}, and since $c(\vec{u}) = v$ belongs to $\sigma$)
	\begin{itemize}
	\item
	there exist exactly $x$ different values $v$ such that 
		\begin{itemize}
		\item
		$c(\vec{u}) = v$ is possible in $W$
		\item
		$v \not \in AV_W(c(\vec{u}))$
		\item
		for one of such $v$, $W\vDash c(\vec{u}) = v$
		\end{itemize}
	\end{itemize}

	Thus condition \textbf{(6)} is satisfied by $I$.
	\item
	Let's consider condition \textbf{(7)} in the definition of $\phi$. 
	Since $I$ satisfies formula (\ref{comp_rempr}), 
	\begin{itemize}
	\item
	$I\vDash \i{RemPr}(c(\vec{u}), x)$ 
	\end{itemize}
	iff 
	\begin{itemize}
	\item
	there exists a $v$ such that $\phi(W)\vDash c(\vec{u}) = v\land \i{PossWithDefPr}(c(\vec{u}) = v)$
	\item
	there exists a random selection rule \eqref{random} along with all pr-atoms \eqref{pr1} associated with it such that
		\begin{itemize}
		\item
		$I \vDash \i{Body}$
		\item
		$y = \sum\limits_{ pr_r(c(\vec{u}) \mvis  v  |  C) = p \in \Pi, \phi(W) \vDash PossWithAssPr_{r, C}(c(\vec{u}) \mvis  v) } p$
		\item
		$x = 1 - y$
		\end{itemize}
	\end{itemize}
	iff (by condition {\bf (5)} and {\bf (2)}, and since $c(\vec{u}) = v$ belongs to $\sigma$)
	\begin{itemize}	
	\item
	there exists a $v$ such that 
		\begin{itemize}
		\item
		$c(\vec{u}) = v$ is possible in $W$
		\item
		$v \not\in AV_W(c(\vec{u}))$
		\item
		$W\vDash c(\vec{u}) = v$
		\end{itemize}
	\item
	$x = 1 - \sum\limits_{ v'\in AV_W(c(\vec{u})) } \i{PossWithAssPr}(W, c(\vec{u}) \mvis  v')$
	\end{itemize}

	Thus condition \textbf{(7)} is satisfied by $I$.	
	\item
	Let's consider condition \textbf{(8)} in the definition of $\phi$. 
	Since $I$ satisfies formula (\ref{comp_totaldefpr}), 
	\begin{itemize}
	\item
	$I\vDash \i{TotalDefPr}(c(\vec{u}), x)$ 
	\end{itemize}
	iff 
	\begin{itemize}
	\item
	$I\vDash \i{RemPr}(c(\vec{u}), x)$
	\item
	$x>0$
	\end{itemize}
	Thus condition \textbf{(8)} is satisfied by $I$.	
	$\newline$
	Now we proved that $I$ is exactly $\phi(W)$, in other words, $I=\phi(I|_{\sigma})$. Thus for every probabilistic stable model $I$ of ${\rm plog2lpmln}(\Pi)$, $I|_{\sigma}$ is a possible world of $\Pi$ and $I = \phi(I|_{\sigma})$. 
	Consequently, $W$ and $\phi(W)$ (or $I|_{\sigma}$ and $I$) are 1-1 correspondent.
	Since $I$ is a probabilistic stable model of $\Pi'$, $P_{\Pi'}(I)>0$. Then $\mu_{\Pi}(I|_{\sigma}) = P_{\Pi'}(I) >0$.
	\end{itemize}
\end{enumerate}
\end{proof}
\qed
\newpage
\section{Full Translation of Monty Hall} \label{sec:examples}
For the P-log program $\Pi$ in \textbf{Example 2}, we showed a part of the translated $\lpmln$ program, ${\rm plog2lpmln}(\Pi)$, in \textbf{Example 3}. The full version of ${\rm plog2lpmln}(\Pi)$ is as follows: ($d \in \{ 1, 2, 3, 4 \}$)
{\small
\[
\ba{l}
// **** ~\tau(\Pi)~ **** \\
\alpha : \sneg \i{CanOpen}(d) \leftarrow \i{Selected} = d \\
\alpha : \sneg \i{CanOpen}(d) \leftarrow \i{Prize} = d \\
\alpha : \i{CanOpen}(d) \leftarrow \no \sneg \i{CanOpen}(d) \\
\\
\alpha : \leftarrow \i{CanOpen}(d), \sneg \i{CanOpen}(d) \\
\alpha : \leftarrow \i{Prize} =d_1, \i{Prize} =d_2, d_1 \neq  d_2 \\
\alpha : \leftarrow \i{Selected} =d_1, \i{Selected} =d_2, d_1 \neq  d_2 \\
\alpha : \leftarrow \i{Open} =d_1, \i{Open} =d_2, d_1 \neq  d_2 \\
\\
\alpha : \i{Prize} =1; \i{Prize} =2; \i{Prize} =3; \i{Prize} =4 \leftarrow \no\ \i{Intervene}(\i{Prize}) \\
\alpha : \i{Selected} =1; \i{Selected} =2; \i{Selected} =3; \i{Selected} =4 \leftarrow \no\ \i{Intervene}(\i{Selected}) \\
\alpha : \i{Open} =1; \i{Open} =2; \i{Open} =3; \i{Open} =4 \leftarrow \no\ \i{Intervene}(\i{Open}) \\
\alpha : \leftarrow \i{Open} =d, \no\ \i{CanOpen}(d), \no\ \i{Intervene}(\i{Open}) \\
\\
\alpha : \i{Obs}(\i{Selected} =1)\\
\alpha : \leftarrow Obs(\i{Selected} =1), \no\ \i{Selected} =1\\

\alpha : \i{Obs}(\i{Open} =2)\\
\alpha : \leftarrow Obs(\i{Open} =2), \no\ \i{Open} =2\\

\alpha : \i{Obs}(Prize \neq 2)\\
\alpha : \leftarrow Obs(Prize \neq 2), \i{Prize} =2\\
\\

// **** \text{ Possible Atoms} **** \\
\alpha : \i{Poss}(\i{Prize} =d) \leftarrow \no\ \i{Intervene}(\i{Prize})\\

\alpha : \i{Poss}(\i{Selected} =d) \leftarrow \no\ \i{Intervene}(\i{Selected})\\

\alpha : \i{Poss}(\i{Open} =d) \leftarrow \i{CanOpen}(d), \no\ \i{Intervene}(\i{Open})\\
\\

// **** \text{ Assigned Probability} **** \\
\alpha : \i{PossWithAssPr}(\i{Prize} = 1) \leftarrow \i{Poss}(\i{Prize} =1) \\

\alpha : \i{AssPr}(\i{Prize} = 1) \leftarrow \i{Prize} = 1, \i{PossWithAssPr}(\i{Prize} = 1) \\

ln(0.3): \bot \leftarrow \no\ \i{AssPr}(\i{Prize} = 1)\\
\\
\alpha : \i{PossWithAssPr}(\i{Prize} = 3) \leftarrow \i{Poss}(\i{Prize} =3) \\

\alpha : \i{AssPr}(\i{Prize} = 3) \leftarrow \i{Prize} = 3, \i{PossWithAssPr}(\i{Prize} = 3) \\

ln(0.2): \bot \leftarrow \no\ \i{AssPr}(\i{Prize} = 3)\\
\\

// **** \text{ Denominator for Default Probability} **** \\
\alpha : \i{PossWithDefPr}(\i{Prize} =d) \leftarrow \i{Poss}(\i{Prize} =d), \no\ \i{PossWithAssPr}(\i{Prize} =d)\\

\alpha : \i{PossWithDefPr}(\i{Selected} =d) \leftarrow \i{Poss}(\i{Selected} =d), \no\ \i{PossWithAssPr}(\i{Selected} =d)\\

\alpha : \i{PossWithDefPr}(\i{Open} =d) \leftarrow \i{Poss}(\i{Open} =d), \no\ \i{PossWithAssPr}(\i{Open} =d)\\
\\
\alpha : \i{NumDefPr}(\i{Prize},x) \leftarrow \i{Prize} = d, \i{PossWithDefPr}(\i{Prize} =d), x = \# count \{ y: \i{PossWithDefPr}(\i{Prize} =y) \} \\

\alpha : \i{NumDefPr}(\i{Selected},x) \leftarrow \i{Selected} = d, \i{PossWithDefPr}(\i{Selected} =d), x = \# count \{ y: \i{PossWithDefPr}(\i{Selected} =y) \} \\

\alpha : \i{NumDefPr}(\i{Open},x) \leftarrow \i{Open} = d, \i{PossWithDefPr}(\i{Open} =d), x = \# count \{ y: \i{PossWithDefPr}(\i{Open} =y) \} \\

ln(\frac{1}{m}): \leftarrow \no\ \i{NumDefPr}(c,m) ~~~~~~~~~~// c \in \{ \i{Prize}, \i{Selected}, \i{Open} \}, m\in \{ 2,3,4\}  \\
\\

// **** \text{ Numerator for Default Probability} **** \\
\alpha : \i{RemPr}(\i{Prize}, 1-x) \leftarrow \i{Prize} = d, \i{PossWithDefPr}(\i{Prize} = d), x = \#{\rm sum}\{0.3:\i{PossWithAssPr}(\i{Prize}\mvis 1); 0.2:\i{PossWithAssPr}(\i{Prize}\mvis 3) \} \\

\alpha : \i{TotalDefPr}(\i{Prize},x) \leftarrow \i{RemPr}(\i{Prize},x), x>0 \\

ln(x): \bot \leftarrow \no\ \i{TotalDefPr}(\i{Prize},x)\\

\alpha : \bot \leftarrow \i{RemPr}(\i{Prize},x), x\leq 0 \\

\ea
\]
}
\newpage
The further translated ASP with Weak Constraints (WC) encoding is as follows: 

{\small
\[
\ba{l}
\% **** Declaration Part **** \\
\i{door}(1..4). \\
\\
\% **** ~\tau(\Pi)~ **** \\
\i{canOpen},(D, f) \leftarrow \i{selected}(D). \\
\i{canOpen},(D, f) \leftarrow \i{prize}(D). \\
\i{canOpen},(D, t) \leftarrow \no\ \i{canOpen},(D, f), \i{door}(D). \\
\\
\leftarrow \i{canOpen},(D, t), \i{canOpen},(D, f). \\
\leftarrow \i{prize}(D_1), \i{prize}(D_2), D_1 \neq  D_2. \\
\leftarrow \i{selected}(D_1), \i{selected}(D_2), D_1 \neq  D_2. \\
\leftarrow \i{open}(D_1), \i{open}(D_2), D_1 \neq  D_2. \\
\\
1\{ \i{prize}(D): \i{door}(D) \}1 \leftarrow \no\ \i{intervene}(\i{prize}). \\
1\{ \i{selected}(D): \i{door}(D) \}1 \leftarrow \no\ \i{intervene}(\i{selected}). \\
1\{ \i{open}(D): \i{door}(D) \}1 \leftarrow \no\ \i{intervene}(\i{open}). \\
\leftarrow \i{open}(D), \no\ \i{canOpen},(D, t), \no\ \i{intervene}(\i{open}). \\
\\
obs(\i{selected}, 1).\\
\leftarrow obs(\i{selected}, 1), \no\ \i{selected}(1).\\
obs(\i{open}, 2).\\
\leftarrow obs(\i{open}, 2), \no\ \i{open}(2).\\
nobs(\i{prize}, 2).\\
\leftarrow nobs(\i{prize}, 2), \i{prize}(2).\\
\\
\% **** \text{ Possible Atoms} **** \\
\i{poss}(\i{prize}, D) \leftarrow \no\ \i{intervene}(\i{prize}).\\
\i{poss}(\i{selected}, D) \leftarrow \no\ \i{intervene}(\i{selected}).\\
\i{poss}(\i{open}, D) \leftarrow \i{canOpen}(D), \no\ \i{intervene}(\i{open}).\\
\\
\% **** \text{ Assigned Probability} **** \\
\i{possWithAssPr}(\i{prize}, 1) \leftarrow \i{poss}(\i{prize}, 1). \\

\i{assPr}(\i{prize}, 1) \leftarrow \i{prize}(1), \i{possWithAssPr}(\i{prize}, 1). \\
\\
\i{possWithAssPr}(\i{prize}, 3) \leftarrow \i{poss}(\i{prize}, 3). \\

\i{assPr}(\i{prize}, 3) \leftarrow \i{prize}(3), \i{possWithAssPr}(\i{prize}, 3). \\
\\
\% **** \text{ Denominator for Default Probability} **** \\
\i{possWithDefPr}(\i{prize}, D) \leftarrow \i{poss}(\i{prize}, D), \no\ \i{possWithAssPr}(\i{prize}, D).\\

\i{possWithDefPr}(\i{selected}, D) \leftarrow \i{poss}(\i{selected}, D), \no\ \i{possWithAssPr}(\i{selected}, D).\\

\i{possWithDefPr}(\i{open}, D) \leftarrow \i{poss}(\i{open}, D), \no\ \i{possWithAssPr}(\i{open}, D).\\
\\
\i{numDefPr}(\i{prize},X) \leftarrow \i{prize}(D), \i{possWithDefPr}(\i{prize}, D), X = \# count \{ Y: \i{possWithDefPr}(\i{prize}, Y) \}. \\

\i{numDefPr}(\i{selected},X) \leftarrow \i{selected}(D), \i{possWithDefPr}(\i{selected}, D), X = \# count \{ Y: \i{possWithDefPr}(\i{selected}, Y) \}. \\

\i{numDefPr}(\i{open},X) \leftarrow \i{open}(D), \i{possWithDefPr}(\i{open}, D), X = \# count \{ Y: \i{possWithDefPr}(\i{open}, Y) \}. \\
\\
\% **** \text{ Numerator for Default Probability} **** \\
\i{remPr}(\i{prize}, Y) \leftarrow \i{prize}(D), \i{possWithDefPr}(\i{prize}, D), X = \#{\rm sum}\{0.3:\i{possWithAssPr}(\i{prize}, 1); 0.2:\i{possWithAssPr}(\i{prize}, 3) \}, Y = 1-X.\\

\i{totalDefPr}(\i{prize}, X) \leftarrow \i{remPr}(\i{prize}, X), X>0.\\

\leftarrow \i{remPr}(\i{prize}, X), X\leq 0. \\

\\
\% **** \text{ Weak Constraints} **** \\
\% \text{note that if we remove this part, we can get all stable models of this program, not just the optimal ones}
\\
:\sneg\ \i{assPr}(\i{prize}, 1). \left[ -ln(0.3) \right] \\
:\sneg\ \i{assPr}(\i{prize}, 3). \left[ -ln(0.2) \right] \\
:\sneg\ \i{numDefPr}(C, M). \left[ -ln(1/M) \right] \\
:\sneg\ \i{totalDefPr}(\i{prize}, X). \left[ -ln(X) \right] \\

\ea
\]
}

\newpage
For both translated $\lpmln$ and WC programs, there are 3 stable models that satisfy all hard rules. The intersection of 3 stable models is shown below, followed by the remaining part of these stable models:

{\small
\[
\ba{l}
\textbf {Intersection of 3 stable models: }\text{(following the syntax in $\lpmln$ encoding)}
\\
\\
\{ \i{Obs}(\i{Selected} =1), \i{Obs}(\i{Open} =2), \i{Obs}(\i{Prize} \neq 2), \\
\i{Selected} =1, \i{Open} =2, \\
\i{CanOpen}(1)=\false , \i{CanOpen}(2)= \true , \\
\i{Poss}(\i{Prize} =1), \i{Poss}(\i{Prize} =2), \i{Poss}(\i{Prize} =3), \i{Poss}(\i{Prize} =4), \\
\i{Poss}(\i{Selected} =1), \i{Poss}(\i{Selected} =2), \i{Poss}(\i{Selected} =3), \i{Poss}(\i{Selected} =4), \\
\i{Poss}(\i{Open} =2), \\
\i{PossWithAssPr}(\i{Prize} = 1), \i{PossWithAssPr}(\i{Prize} = 3), \\
\i{PossWithDefPr}(\i{Prize} =2), \i{PossWithDefPr}(\i{Prize} =4), \\
\i{PossWithDefPr}(\i{Selected} =1), \i{PossWithDefPr}(\i{Selected} =2), \i{PossWithDefPr}(\i{Selected} =3), \i{PossWithDefPr}(\i{Selected} =4), \\
\i{PossWithDefPr}(\i{Open} =2), \\
\i{NumDefPr}(\i{Selected},4)
\}
\ea
\]
}

{\small
\[
\ba{l}
I_1 = \{ \i{Prize} =1, \i{CanOpen}(3)= \true, \i{CanOpen}(4)= \true, \\
\hspace{28pt} \i{AssPr}(\i{Prize} =1), \i{NumDefPr}(\i{Open},3), \\
\hspace{28pt} \i{Poss}(\i{Open} =3), \i{Poss}(\i{Open} =4), \\
\hspace{28pt} \i{PossWithDefPr}(\i{Open} =3), \i{PossWithDefPr}(\i{Open} =4) \} \\
\\
I_2 = \{ \i{Prize} =3, \i{CanOpen}(3)= \false, \i{CanOpen}(4)= \true, \\
\hspace{28pt} \i{AssPr}(\i{Prize} = 3), \i{NumDefPr}(\i{Open},2), \\
\hspace{28pt} \i{Poss}(\i{Open} =4), \\
\hspace{28pt} \i{PossWithDefPr}(\i{Open} =4) \} \\
\\
I_3 = \{ \i{Prize} =4, \i{CanOpen}(3)= \true, \i{CanOpen}(4)= \false, \\
\hspace{28pt} \i{NumDefPr}(\i{Prize}, 2), \i{TotalDefPr}(\i{Prize}, 0.5), \i{NumDefPr}(\i{Open},2), \\
\hspace{28pt} \i{Poss}(\i{Open} =3), \\
\hspace{28pt} \i{PossWithDefPr}(\i{Open} =3), \\
\hspace{28pt} \i{RemPr}(\i{Prize}, 0.5)\} \\
\ea
\]
}

The unnormalized weight $\omega (I_i)$ of each stable model $I_i$ is shown below:

\[
\ba{l l}
\omega (I_1) &= \omega(\i{NumDefPr}(\i{Selected},4))\times \hspace{12pt} \omega(\i{AssPr}(\i{Prize} = 1)) \times \omega(\i{NumDefPr}(\i{Open},3)) \\ \\
 &= \frac{1}{4} ~~~~\times~~~~ 
0.3 \times \frac{1}{3} = \frac{1}{40} \\
\\
\omega (I_2) &= \omega(\i{NumDefPr}(\i{Selected},4))\times \hspace{12pt} \omega(\i{AssPr}(\i{Prize} = 3)) \times \omega(\i{NumDefPr}(\i{Open},2)) \\ \\
 &= \frac{1}{4} ~~~~\times~~~~ 
0.2 \times \frac{1}{2} = \frac{1}{40} \\
\\
\omega (I_3) &= \omega(\i{NumDefPr}(\i{Selected},4))\times \hspace{12pt} \omega(\i{NumDefPr}(\i{Prize}, 2)) \times \omega(\i{TotalDefPr}(\i{Prize}, 0.5)) \times \omega(\i{NumDefPr}(\i{Open},2)) \\ \\
 &= \frac{1}{4} ~~~~\times~~~~ 
\frac{1}{2} \times 0.5 \times \frac{1}{2} = \frac{1}{32}
\ea
\]


\begin{thebibliography}{}

\bibitem[\protect\citeauthoryear{Bach \bgroup et al\mbox.\egroup
  }{2015}]{bach15hinge}
Bach, S.~H.; Broecheler, M.; Huang, B.; and Getoor, L.
\newblock 2015.
\newblock Hinge-loss markov random fields and probabilistic soft logic.
\newblock arXiv:1505.04406 [cs.LG].

\bibitem[\protect\citeauthoryear{Balai and
  Gelfond}{2016}]{balai16ontherelationship}
Balai, E., and Gelfond, M.
\newblock 2016.
\newblock On the relationship between {P}-log and ${{\rm L}{\rm P}^{{\rm M}{\rm
  L}{\rm N}}}$.
\newblock In {\em Proceedings of International Joint Conference on Artificial
  Intelligence ({IJCAI})}.

\bibitem[\protect\citeauthoryear{Baral, Gelfond, and
  Rushton}{2009}]{baral09probabilistic}
Baral, C.; Gelfond, M.; and Rushton, J.~N.
\newblock 2009.
\newblock Probabilistic reasoning with answer sets.
\newblock {\em TPLP} 9(1):57--144.

\bibitem[\protect\citeauthoryear{Buccafurri, Leone, and
  Rullo}{2000}]{buccafurri00enhancing}
Buccafurri, F.; Leone, N.; and Rullo, P.
\newblock 2000.
\newblock Enhancing disjunctive datalog by constraints.
\newblock {\em Knowledge and Data Engineering, IEEE Transactions on}
  12(5):845--860.

\bibitem[\protect\citeauthoryear{Calimeri \bgroup et al\mbox.\egroup
  }{2013}]{calimeri13aspcore2}
Calimeri, F.; Faber, W.; Gebser, M.; Ianni, G.; Kaminski, R.; Krennwallner, T.;
  Leone, N.; Ricca, F.; and Schaub, T.
\newblock 2013.
\newblock {A}{S}{P}-{C}ore-2 input language format.

\bibitem[\protect\citeauthoryear{De~Raedt, Kimmig, and
  Toivonen}{2007}]{deraedt07problog}
De~Raedt, L.; Kimmig, A.; and Toivonen, H.
\newblock 2007.
\newblock {P}rob{L}og: A probabilistic prolog and its application in link
  discovery.
\newblock In {\em IJCAI}, volume~7,  2462--2467.

\bibitem[\protect\citeauthoryear{Ferraris, Lee, and
  Lifschitz}{2011}]{ferraris11stable}
Ferraris, P.; Lee, J.; and Lifschitz, V.
\newblock 2011.
\newblock Stable models and circumscription.
\newblock {\em Artificial Intelligence} 175:236--263.

\bibitem[\protect\citeauthoryear{Ferraris}{2011}]{ferraris11logic}
Ferraris, P.
\newblock 2011.
\newblock Logic programs with propositional connectives and aggregates.
\newblock {\em ACM Transactions on Computational Logic (TOCL)} 12(4):25.

\bibitem[\protect\citeauthoryear{Gelfond and Lifschitz}{1988}]{gel88}
Gelfond, M., and Lifschitz, V.
\newblock 1988.
\newblock The stable model semantics for logic programming.
\newblock In Kowalski, R., and Bowen, K., eds., {\em Proceedings of
  International Logic Programming Conference and Symposium},  1070--1080.
\newblock MIT Press.

\bibitem[\protect\citeauthoryear{Harrison, Lifschitz, and
  Yang}{2014}]{harrison14thesemantics}
Harrison, A.~J.; Lifschitz, V.; and Yang, F.
\newblock 2014.
\newblock The semantics of gringo and infinitary propositional formulas.
\newblock In {\em Principles of Knowledge Representation and Reasoning:
  Proceedings of the Fourteenth International Conference, {KR} 2014}.

\bibitem[\protect\citeauthoryear{Lee and Meng}{2012}]{lee12stable1}
Lee, J., and Meng, Y.
\newblock 2012.
\newblock Stable models of formulas with generalized quantifiers (preliminary
  report).
\newblock In {\em Technical Communications of the 28th International Conference
  on Logic Programming},  61--71.

\bibitem[\protect\citeauthoryear{Lee and Wang}{2016}]{lee16weighted}
Lee, J., and Wang, Y.
\newblock 2016.
\newblock Weighted rules under the stable model semantics.
\newblock In {\em Proceedings of International Conference on Principles of
  Knowledge Representation and Reasoning (KR)}, 145--154.

\bibitem[\protect\citeauthoryear{Lee, Meng, and Wang}{2015}]{lee15markov}
Lee, J.; Meng, Y.; and Wang, Y.
\newblock 2015.
\newblock Markov logic style weighted rules under the stable model semantics.
\newblock In Technical Communications of the 31st International Conference on
  Logic Programming.

\bibitem[\protect\citeauthoryear{Nickles and
  Mileo}{2014}]{nickles14probabilistic}
Nickles, M., and Mileo, A.
\newblock 2014.
\newblock Probabilistic inductive logic programming based on answer set
  programming.
\newblock In {\em 15th International Workshop on Non-Monotonic Reasoning (NMR
  2014)}.

\bibitem[\protect\citeauthoryear{Pearl}{2000}]{pearl00causality}
Pearl, J.
\newblock 2000.
\newblock {\em Causality: models, reasoning and inference}, volume~29.
\newblock Cambridge Univ Press.

\bibitem[\protect\citeauthoryear{Richardson and
  Domingos}{2006}]{richardson06markov}
Richardson, M., and Domingos, P.
\newblock 2006.
\newblock Markov logic networks.
\newblock {\em Machine Learning} 62(1-2):107--136.


\bibitem[\protect\citeauthoryear{Bartholomew, Michael, and Lee}{2013}]{bartholomew13onthestable}
Bartholomew, Michael and Lee, Joohyung
\newblock 2013.
\newblock On the stable model semantics for intensional functions.
\newblock Cambridge University Press. 

\bibitem[\protect\citeauthoryear{Ferraris, Lee, Lifschitz, and Palla}{2009}]{ferr09b}
Paolo Ferraris and Joohyung Lee and Vladimir Lifschitz and Ravi Palla
\newblock 2009.
\newblock Symmetric splitting in the general theory of stable models.
\newblock Proceedings of International Joint Conference on Artificial Intelligence (IJCAI). 

\end{thebibliography}
\end{document}